\documentclass[letterpaper,11pt]{article}
\usepackage[round,authoryear]{natbib}
\usepackage[margin=1in]{geometry}
\usepackage{amsmath,amssymb}

\usepackage{amsthm}
\usepackage{mathtools}
\usepackage{hyperref}
\usepackage[nameinlink]{cleveref}    
\usepackage{enumitem}
\setitemize{noitemsep,topsep=0pt,parsep=0pt,partopsep=0pt,leftmargin=*}
\usepackage{parskip} 
\usepackage{xcolor}
\usepackage{titlesec}
\usepackage{tikz}
\usepackage{dsfont}
\usepackage{booktabs}   
\usepackage{ragged2e}   
\usepackage{arydshln}

\usepackage[utf8]{inputenc} 
\usepackage[T1]{fontenc}    
\usepackage{url}            
\usepackage{booktabs}       
\usepackage{amsfonts}       
\usepackage{nicefrac}       
\usepackage{microtype}      

\usetikzlibrary{arrows.meta}
\usetikzlibrary{shapes,arrows,shadows,positioning}
\usetikzlibrary{backgrounds}
\usetikzlibrary{arrows,chains,matrix,positioning,scopes} 
\newtheorem{defn}{Definition}
\newtheorem{assum}{Assumption}
\newtheorem{lem}{Lemma}
\newtheorem{thm}{Theorem}

\newtheorem{prop}{Proposition}
\newtheorem{remark}{Remark}
\newenvironment{ack}{\section*{Acknowledgments}}{}
\hypersetup{
  colorlinks=true,
  linkcolor=blue,
  filecolor=magenta,
  urlcolor=blue,
  citecolor=blue,
}
\crefname{equation}{}{}
\Crefname{equation}{Equation}{Equations}
\crefname{assum}{Ass.}{Ass.}
\Crefname{assum}{Assumption}{Assumptions}
\crefname{prop}{Prop.}{Prop.}
\Crefname{prop}{Proposition}{Propositions}
\crefname{lem}{Lem.}{Lem.}
\Crefname{lem}{Lemma}{Lemmas}
\crefname{defn}{Defn.}{Defn.}
\Crefname{defn}{Definition}{Definitions}
\crefname{cor}{Cor.}{Cor.}
\Crefname{cor}{Corollary}{Corollaries}
\crefname{thm}{Thm.}{Thm.}
\Crefname{thm}{Theorem}{Theorems}
\creflabelformat{equation}{(#2#1#3)}

\usepackage{glossaries}
\newacronym{mpc}{MPC}{model predictive control}
\newacronym{ols}{OLS}{ordinary least-squares}
\newacronym{rv}{RV}{random variable}
\newacronym{lqr}{LQR}{linear quadratic regulator}
\newacronym{iid}{i.i.d.}{independently and identically distributed}
\newacronym{arx}{ARX}{autoregressive with exogenous input}
\newacronym{var}{VAR}{vector autoregressive}
\newcommand{\bA}{\boldsymbol{A}}
\newcommand{\bB}{\boldsymbol{B}}
\newcommand{\bC}{\boldsymbol{C}}
\newcommand{\bD}{\boldsymbol{D}}
\newcommand{\bE}{\boldsymbol{E}}
\newcommand{\bxi}{\boldsymbol{\xi}}

\newcommand{\bigO}{\mathcal{O}}

\newcommand{\poly}{\text{poly}}
\newcommand{\polyGburn}{\poly_3}
\newcommand{\polyGbound}{\poly_4}
\newcommand{\polyEburn}{\poly_5}
\newcommand{\polyEbound}{\poly_6}
\newcommand{\polyPburn}{\poly_1}
\newcommand{\polyPbound}{\poly_2}

\title{Learning~Linear~Systems under~Heavy-Tailed~Noise: A~Non-Asymptotic~Analysis from~A~Single~Trajectory}

\author{%
  Xiaomian Yang \\
  Department of Chemical Engineering \\
  Massachusetts Institute of Technology
  Cambridge, MA 02139 \\
  \texttt{xiaomian@mit.edu} \\
  \and
  Sungho Shin \\
  Department of Chemical Engineering \\
  Massachusetts Institute of Technology \\
  Cambridge, MA 02139 \\
  \texttt{sushin@mit.edu} \\
  \thanks{Corresponding Author: \url{sushin@mit.edu}}
}

\begin{document}

\maketitle

\begin{abstract}
  We establish non-asymptotic sample complexity bounds 
  for the least-squares estimation of vector autoregressive models 
  for exponentially stable systems
  with heavy-tailed noise based on a single observed trajectory.
  By assuming i.i.d. noise, bounded noise covariance, 
  and persistent excitation, we show that the estimation error 
  is $\widetilde{\bigO}(r^{1/2}T^{-1/2+1/p})$ under bounded $p$th moment for $p > 2$, 
  where $T$ is the number of samples, $r$ is the noise dimension, and $\widetilde{\bigO}(\cdot)$ 
  hides logarithmic terms.
  We also introduce a unifying approach to sample complexity analysis 
  applicable to broad classes of noise distributions 
  and showcase this by deriving error bounds
  for sub-exponential and sub-Gaussian noise distributions.
  Finally, we specialize our analysis to autoregressive models with exogenous inputs 
  and show that the dimension factor of the error bound 
  is independent of the model order.
\end{abstract}


\section{Introduction}
System identification has received renewed attention 
in recent years due to its central role in data-driven model-based control 
and decision-making algorithms 
\citep{dean2017samplecomplexitylinearquadratic,Simchowitz2018-if,oymakNonasymptoticIdentificationLTI2019,
Tsiamis2022-pb,ziemannTutorialNonAsymptoticTheory2024,Faradonbeh2017-ju,Bakshi2023-az,Kanakeri2024-xh,Sattar2025-lr}. 
In many applications---ranging from robotics \citep{Rosolia2018-xx} to process control \citep{qin1997overview}---
one must learn a dynamical model from data in order to design controllers 
that are both sample-efficient and constraint-aware. 
In particular, \gls*{mpc} relies explicitly 
on system models often learned via system identification
to make predictions and optimize control actions
\citep{Rawlings2017-nw,qin1997overview}. 
This paradigm contrasts with direct policy learning approaches
\citep{Fazel2018-hn,hu2023toward,Sutton2018-ji}, 
where control policies are learned end-to-end without explicitly learning the underlying system dynamics. 
When data are limited and safety or operational constraints must be respected, model-based control---learning the system model with system identification and subsequently synthesizing a controller (e.g., robust control or \gls*{mpc})---is often the only viable approach.

Classical system identification theory has primarily focused on asymptotic properties of estimators, 
establishing consistency and convergence as the number of samples tends to infinity \citep{ljung1998system}. 
While these results establish basic properties of the estimator, they are insufficient for understanding performance 
in the finite-sample regime relevant to learning-based control. 
In order to analyze end-to-end performance and to compare different paradigms,
non-asymptotic guarantees of estimation error for a finite number of samples are crucial, 
and a growing body of recent work has developed non-asymptotic sample complexity bounds 
for linear system identification \citep{dean2017samplecomplexitylinearquadratic,Simchowitz2018-if,
oymakNonasymptoticIdentificationLTI2019,ziemannTutorialNonAsymptoticTheory2024}.

Most existing non-asymptotic results on \gls*{ols} rely on 
\emph{strong assumptions regarding the noise distribution}. 
In particular, the majority of current analyses assume that the system noise
is sub-Gaussian 
\citep{dean2017samplecomplexitylinearquadratic,
Simchowitz2018-if,oymakNonasymptoticIdentificationLTI2019,
ziemannTutorialNonAsymptoticTheory2024}
while few focus on the sub-Weibull class\citep{Faradonbeh2017-ju},
which excludes distributions with polynomial tail bounds 
such as Pareto or Student's t-distributions. 
Heavy-tailed disturbances arise naturally in applications 
involving outliers, intermittent shocks, or imperfect sensing and actuation,
such as power grid blackout frequency that is often modeled by the Pareto distribution \citep{Nair2022-wg}. 
The existing theoretical guarantees based on light-tailed noise distributions 
and exponentially decaying heavy-tailed distributions are inapplicable in such scenarios.
This motivates the central question of this paper: 
\emph{Does model-based control remain statistically viable when system noise exhibits heavy tails?}

We address this question by studying the non-asymptotic sample complexity of \gls*{ols} estimators 
for linear dynamical systems
under heavy-tailed noise distributions.
Specifically, we focus on
(i) \gls*{ols} estimator, (ii) identification from single trajectory, 
(iii) exponentially stable system, 
and (iv) general heavy-tailed noise distributions for bounded $p$th moment with $p>2$.
Rationale for these settings is as follows:
\gls*{ols} is often employed in practice as the go-to method
even when the noise covariance is unknown due to its computational tractability 
and is the basis for a variety of practical algorithms \citep{qin1997overview}.
Single trajectory is often the most realistic data collection scenario in practice, 
as one typically has access to only one realization of the system evolution.
Most controllers designed by engineers are exponentially stable 
for reliability and performance guarantees,
such as power system frequency and voltage control, process control, and robotics.
Finally, relaxing light-tailed noise assumptions is crucial 
for understanding the robustness of model-based approaches in real-world settings 
where the system is subject to heavy-tailed noises.

\paragraph{Contributions}
This paper provides a non-asymptotic sample complexity analysis 
of \gls*{ols}-based linear system identification on the \gls*{var} model
for exponentially stable system under heavy-tailed noise distributions
with bounded $p$th moment with $p>2$.
We present a unifying approach for non-asymptotic sample complexity analysis 
that accommodates a broad class of noise distributions 
and yields explicit finite-sample error bounds expressed in terms of general system properties.
Finally, we derive non-asymptotic estimation error upper bounds 
for heavy-tailed distributions with bounded $p$th moment for $p>2$
and demonstrate its applicability to
sub-exponential
and sub-Gaussian distributions.
To the best of our knowledge, this is the first work to provide 
comprehensive non-asymptotic sample complexity guarantees 
for (i) \gls*{ols} estimators, (ii) from single trajectory, 
(iii) for exponentially stable systems,
(iv) under general heavy-tailed noise distributions.

By assuming exponential stability, 
persistent excitation, \gls*{iid} noise,
and bounded propagation of disturbances,
we show that the estimation error is bounded by
$\widetilde{\bigO}(r^{1/2}T^{-1/2+1/p})$ under bounded $p$th moment with $p> 2$,
where $T$ is the sample size and $r$ is the noise dimension.
Notably, our analysis decouples the noise and state dimensions
to address certain classes of rank-deficient systems,
where our bound is expressed entirely in $r$ 
and can be smaller than the state dimension.
Our analysis provides an
explicit characterization not only of the error scale but also of the dependence of 
system-specific constants that characterize stability, persistent excitation, and covariance bound.
This clearly reveals the effect of these properties on the parameter estimation error.
Our analysis demonstrates that for $p>2$ distributions, 
the estimation error can be controlled by increasing the number of samples, 
albeit at a slower rate than in the sub-Gaussian case.
\emph{This suggests that model-based approaches remain statistically viable 
even in the presence of moderately heavy-tailed (bounded $p$th moment with $p>2$) noise distributions}.

Our technique uses a fast mixing argument (via exponential stability) 
and a blocking strategy by introducing stacking of noise vectors,
which improves upon previous works \citep{Yu1994-uj} 
by removing temporal correlation completely.
The estimation error can then be expressed 
in terms of the empirical covariance tail bound of the augmented noise distribution, 
where standard concentration inequalities for independent random variables can be applied.
Another contribution, also at a more technical level, 
is deriving a tighter sample complexity bound for \gls*{arx} model identification.
We specialize our general analysis to \gls*{arx} models and show that, 
under suitable excitation and stability conditions, 
the dimension factor of the sample complexity depends only on the
combined state and input dimension and is independent of the autoregressive order.
This result is an improvement over existing bounds 
that scale with the order of the \gls*{arx} model \citep{ziemannTutorialNonAsymptoticTheory2024}, 
where the complexity scales with $q^{1/2}$ and $q$ is the order of the \gls*{arx} model.

\paragraph{Related Work}

Previous works dedicated 
to non-asymptotic sample complexity analysis
have primarily focused on sub-Gaussian noise
and generalized
results for exponentially stable systems to marginally stable systems. 
\cite{dean2017samplecomplexitylinearquadratic} first explored 
convergence of finite samples with
multi-trajectory inputs.
This was extended to single trajectory by
\cite{oymakNonasymptoticIdentificationLTI2019}
for the Markov parameter 
of the state-space model.
\cite{Tsiamis2022-pb} relaxed the noise distribution to i.i.d. sub-Gaussian noise,
and is generalized by \cite{ziemannTutorialNonAsymptoticTheory2024,Simchowitz2018-if}
to non-explosive or marginally stable systems.
In particular, \cite{Simchowitz2018-if} avoids the mixing time argument
and shows that marginally stable systems can be easier to estimate
due to better signal-to-noise ratio.
In comparison,
our work focuses on exponentially stable systems
and heavy-tailed noise
and does not address marginally stable systems.

Theoretical guarantees for sample complexity 
of systems affected by heavy-tailed distributions are limited.
Perhaps most similar to our work is \cite{Faradonbeh2017-ju}, 
which explored \gls*{ols} with sub-Weibull noise on a stable system, 
and obtained $\widetilde{\bigO}(T^{-1/2}d^{1/2})$.
However, their analysis only addressed
distributions with exponentially decaying tails
and not polynomially decaying tails.
Our work extends the scope of noise distributions, 
such as Pareto and Student's t-distributions, 
which can be captured by bounded $p$th moment assumptions.
\cite{Sattar2025-lr} also conducted analysis for \gls*{ols}
and relaxed the noise distribution
to only bounded covariance
but made a restrictive bounded input assumption.
Recent work by \cite{Bakshi2023-az} considered noise distributions
with only bounded $4$th moment
and guaranteed an upper bound of $\widetilde{\bigO}(T^{-1/2})$. The estimator studied therein 
requires the solution of a constraint system that is only guaranteed 
to be feasible under non-trivial conditions, thus limiting its applicability.
\cite{Kanakeri2024-xh}
established a similar bound 
with a robust linear least-squares algorithm,
but only for multi-trajectory data.

We summarize existing results and compare to our work below. Let 
$N$ be the number of trajectories, $d$ the state dimension, 
$r$ the noise dimension, $T$ the number of samples,
$d_y$ the output dimension, $d_u$ the input dimension, $q$ the order of the \gls*{arx} model,
and $\Delta$ the estimation error.

\begin{itemize}
\item \cite{dean2017samplecomplexitylinearquadratic}: coarse-ID; \gls*{ols} estimator; 
multi-trajectory; Gaussian noise; $\Delta=\bigO(d^{1/2}N^{-1/2})$.
\item \cite{oymakNonasymptoticIdentificationLTI2019}: Markov parameter of state-space model; \gls*{ols} estimator; 
single-trajectory; stable system; Gaussian noise; $\Delta=\widetilde{\bigO}(d^{1/2}T^{-1/2})$.
\item \cite{ziemannTutorialNonAsymptoticTheory2024}: ARX; \gls*{ols} estimator; 
single-trajectory; marginally stable system; sub-Gaussian noise; $\Delta=\widetilde{\bigO}((q(d_y+d_u))^{1/2}T^{-1/2})$.
\item \cite{Faradonbeh2017-ju}: \gls*{var}; \gls*{ols} estimator; single-trajectory; stable system; 
sub-Weibull noise; $\Delta=\widetilde{\bigO}(d^{1/2}T^{-1/2})$.
\item \cite{Bakshi2023-az}: Markov parameter of state-space model; an estimator based on a custom constraint system; 
single-trajectory; marginally stable system; bounded $4$th moment; $T = \Omega(\log^3(1/\Delta)/\Delta^2)$.
\item \cite{Kanakeri2024-xh}: \gls*{var}; robust \gls*{ols} estimator; 
multi-trajectory; bounded $4$th moment; $\Delta=\bigO(d^{3/2}T^{-1/2})$.
\item \textbf{This work}: \gls*{var} and ARX; \gls*{ols} estimator; 
single-trajectory; stable system; under bounded $p$th moment for $p>2$; $\Delta=\widetilde{\bigO}(T^{-1/2+1/p}r^{1/2})$.
\end{itemize}

\paragraph{Organization}
The rest of the paper is organized as follows:
\Cref{sec:setting} introduces the problem setting.
\Cref{sec:unifying} presents a unifying lemma for non-asymptotic sample complexity analysis.
\Cref{sec:examples} applies the unifying lemma to bounded $p$th moment distributions.
\Cref{sec:arx} specializes the results to \gls*{arx} models.
\Cref{sec:conclusion} concludes the paper.
All the proofs are deferred to \Cref{sec:main-thm-proof,sec:thm-subgauss,sec:thm-subexp,sec:thm-pthmom,sec:arx-appendix,sec:thm-arx}.
Numerical results are shown in \Cref{sec:numerical}.

\paragraph{Notation}
We denote the set of real numbers and integers by $\mathbb{R}$ and $\mathbb{Z}$, respectively. 
The set of positive (nonnegative) real numbers and integers 
are denoted by $\mathbb{R}_{>0}$ ($\mathbb{R}_{\geq 0}$) and $\mathbb{Z}_{> 0}$ ($\mathbb{Z}_{\geq 0}$), respectively.
Vector or matrix $p$-norms are represented by $\|\cdot\|_p$, and we let $\|\cdot\| = \|\cdot\|_2$.
$\lambda_{\text{min}}(A)$ is the minimum eigenvalue of matrix $A$
while $\lambda_{\text{max}}(A)$ is the maximum eigenvalue.
$\mathds{1}(\cdot)$ is the indicator function.
$\mathbb{S}^{d-1}$ denotes the unit sphere in $\mathbb{R}^d$.
For vectors $x_1,x_2,\cdots,x_n$, 
we denote the concatenation of these vectors 
as $(x_1,x_2,\cdots,x_n) = [x_1^\top\;x_2^\top\;\cdots\;x_n^\top]^\top$.
We denote $\log^p(x) = (\log(x))^p$ for $p\in\mathbb{R}$.

\section{Problem Formulation}\label{sec:setting}

We consider the \gls*{var} model, 
a linear dynamical system of the following form:
\begin{align}\label{eqn:sys}
  \begin{aligned}[t]
    z_{0} & \;=\; \Psi\xi_0                                         \\
    z_{t} & \;=\; \Theta z_{t-1} + \Psi\xi_t,\quad t=1,2,\cdots, T,
  \end{aligned}
\end{align}
where
$d,r,T\in \mathbb{Z}_{>0}$ are system dimension, noise dimension, and trajectory length,
$\Theta\in\mathbb{R}^{d\times d}$ and $\Psi \in\mathbb{R}^{d \times r}$ 
are unknown system parameters,
$\{z_t\}_{t=0}^{T}$ is the sequence of system states in $\mathbb{R}^d$,
and $\{\xi_t\}_{t=0}^{T}$ is a sequence of random noise vectors in $\mathbb{R}^r$.
In system identification and model-based control, estimating $\Theta$ is the first step towards 
identifying system dynamics and designing control policies.
Thus, we focus on the problem of identifying unknown $\Theta$ 
from the observations $\{z_t\}_{t=0}^{T}$ 
using the \gls*{ols} estimator of the following form:
\begin{align}\label{eqn:est}
  \widehat{\Theta} \in \arg\min_{\Theta} \sum_{t=1}^{T} \left\|z_{t} - \Theta z_{t-1}\right\|^2.
\end{align}
We will derive an upper bound on the estimation error $\|\widehat{\Theta}-\Theta\|$ 
under a finite trajectory length $T$ that holds with probability at least $1-\delta$ 
for some user-specified failure probability
$\delta\in(0,1)$, 
expressed in terms of the trajectory length $T$, noise dimension $r$, 
failure probability $\delta$, and system-specific constants.

\paragraph{Generality of the Formulation}
System identification is often studied under \gls*{arx} model formulations, 
which include control (exogenous input) terms.
In \Cref{sec:arx}, we will show the results derived for the estimator in \cref{eqn:est} 
can be directly applied to identification of \gls*{arx} model with exogenous inputs
sampled from i.i.d. zero-mean distributions.
Another popular setting is the state-space model, where the system states are not directly measured.
One can perform ARX model identification as an intermediate step to identify the Markov parameters
and recover the state-space model from the identified ARX model 
using subspace identification methods, such as the Ho-Kalman method
\citep{Van_Overschee1994-th,Qin2006-iz,oymakNonasymptoticIdentificationLTI2019}.
We do not explicitly analyze the error bound for state-space models 
since it can be derived from the \gls*{arx} model.

\section{Unifying Lemma for Error Bounds}\label{sec:unifying}

We present the unifying lemma 
for analyzing the finite-sample error bound
of the least squares estimator in \cref{eqn:est}.
We derive the error bound 
for the estimate based on a set of sufficient conditions
discussed in \Cref{assum:main}, 
such as \gls*{iid} and zero-mean noise, exponential stability, bounded effect of noise, 
and persistent excitation.
We also make an assumption on the empirical covariance tail bound of the augmented noise vector
discussed in \Cref{assum:tail}, 
which is applicable to various noise distributions.
We do not assume specific noise distributions 
for $\xi_t$ in this section and defer that to
\Cref{sec:examples},
where we will establish sufficient conditions for \Cref{assum:tail} 
under specific noise distributions.
The following are the main assumptions regarding the properties of the system in \cref{eqn:sys}.
\begin{assum}\label{assum:main}
  Consider the system in \cref{eqn:sys} with $d,r,T\in \mathbb{Z}_{>0}$, $\Theta\in\mathbb{R}^{d\times d}$, $\Psi\in\mathbb{R}^{d\times r}$, 
  and $\xi_t \in \mathbb{R}^r$.
  There exist $L,\beta,\sigma \in \mathbb{R}_{>0}$,
  $\alpha\in(0,1)$, and $\overline{t}_c\in\mathbb{Z}_{>0}$, such that the following conditions hold:
  \begin{enumerate}[label=(\alph*),leftmargin=*,ref=\theassum(\alph*)]
  \item\label[assum]{assum:iid} \textbf{\Gls*{iid}, zero-mean, unit covariance}: 
    \begin{align}\label{eqn:iid}
      \xi_0, \xi_1, \cdots, \xi_T \text{ are \gls*{iid}}
      , \quad
      \mathbb{E}[\xi_t]  = 0, \quad \mathbb{E}[\xi_t \xi_t^\top] = I , \quad \forall t=0,1,\cdots, T,
    \end{align}
  \item\label[assum]{assum:stab} \textbf{Stability}: 
    \begin{align}\label{eqn:stab}
      \|\Theta^t\| \leq L\alpha^t,\quad \forall t=0,1,\cdots
    \end{align}
  \item\label[assum]{assum:bounded-cov} \textbf{Bounded effect of noise}:
    \begin{align}\label{eqn:bounded-cov}
      \Psi\Psi^\top \preceq \sigma^2 I 
    \end{align}
  \item\label[assum]{assum:ctrb-sigma} \textbf{Persistent excitation}:
    \begin{align}\label{eqn:ctrb-sigma}
      \begin{bmatrix}
        \Psi & \Theta \Psi & \cdots &\Theta^{\overline{t}_c-1} \Psi
      \end{bmatrix}
      \begin{bmatrix}
        \Psi & \Theta \Psi & \cdots &\Theta^{\overline{t}_c-1} \Psi
      \end{bmatrix}^\top
      \succeq
      \beta^2 I
    \end{align}
  \end{enumerate}
\end{assum}

In \Cref{assum:iid},
  the zero-mean assumption is without loss of generality, 
  as any non-zero mean can be absorbed into the system dynamics 
  by redefining the state variable appropriately if $\Theta$ is stable. 
  Unit-covariance is assumed (almost) without loss of generality, 
  as any positive definite covariance matrix can be absorbed into $\Psi$ 
  by redefining $\Psi \leftarrow \Psi \Sigma^{1/2}$ where $\Sigma$ is the covariance matrix of $\xi_t$.
\Cref{assum:stab} guarantees the system exponentially converges to the origin in the absence of noise
  \citep{oymakNonasymptoticIdentificationLTI2019}.
  This condition is satisfied when $|\lambda|<1$, for all eigenvalues $\lambda$ of $\Theta$.
  A weaker form is the non-explosive system 
  employed by \cite{Simchowitz2018-if,ziemannTutorialNonAsymptoticTheory2024,Bakshi2023-az}, 
  where $|\lambda|\leq 1$.
  In this work, we assume the stronger condition of exponential stability, 
  as it allows us to express the estimation error bound 
  in the form of augmented noise vectors with bounded length (see \Cref{assum:tail}) 
  by using the mixing time of the system.
\Cref{assum:bounded-cov} bounds the covariance of the \emph{effect of noise} $\Psi\xi_t$; 
  that is, $\mathbb{E}\left[\Psi\xi_t(\Psi\xi_t)^\top\right]=\Psi\Psi^\top\preceq \sigma^2 I$.
\Cref{assum:ctrb-sigma} ensures persistent excitation 
\citep{Van-Overschee1996-ee}, 
which requires that the controllability matrix 
of the system up to time $\overline{t}_c$ 
is full rank with the minimum singular value lower bounded by $\beta$.
This means that, by selecting excitation signals appropriately, 
the system can reach any state in $\mathbb{R}^d$ within $\overline{t}_c$ time steps.
The input signals can excite all modes of the system, 
so the system is informative enough to identify the system parameters.

Our next assumption introduces the covariance tail bound
for the \emph{augmented} noise defined below.
\newcommand{\tail}{\sum_{t=1}^T \bxi_{t:t-k+1}\bxi_{t:t-k+1}^\top  - \mathbb{E}\left[\sum_{t=1}^T \bxi_{t:t-k+1}\bxi_{t:t-k+1}^\top\right]}
\begin{assum}[\textbf{Bounded Tail Probability}] \label[assum]{assum:tail}
  Consider the sequence of \glspl*{rv} $\{\xi_t\}_{t=0}^{T}$ in \cref{eqn:sys}. 
  There exist $\epsilon \in \mathbb{R}_{>0}$, $\delta \in (0,1)$, and $k\in\mathbb{Z}_{>0}$ such that the following holds:
  \begin{align}\label{eqn:tail}
      \mathbb{P}\left[\left\|
      \tail
      \right\|\geq \epsilon
      \right] \leq \delta,
    \end{align}
    where
    $\bxi_{t:t-k+1}=
    (
      \xi_{t}   ,
      \xi_{t-1} ,
      \cdots    ,
      \xi_{t-k+1} )$,
      and $\xi_{t}=0$ if $t<0$ or $t>T$.
\end{assum}
The order of augmentation $k$ is a tunable variable, 
which serves as the mixing time of the system in \Cref{lem:main}.
This assumption is a generalization of the common concentration inequalities 
for empirical covariance matrices 
\cite[chapter 6]{wainwrightHighDimensionalStatisticsNonAsymptotic2019}.
\Cref{assum:tail} ensures that the difference 
between the empirical covariance and the true covariance 
is bounded with high probability.
In \Cref{sec:examples}, we will derive $\epsilon$ explicitly 
for different classes of noise distributions 
by leveraging distribution-related properties 
and applying concentration inequalities.
The bound $\epsilon$
depends on the failure probability $\delta$, 
sample size $T$, noise dimension $r$,
and distribution-dependent variables for $\xi_t$ (e.g., $p$th moment).

\newcommand{\epsbound}{\frac{\beta^2}{8}\left(\frac{\sigma^2 L^2}{1-\alpha^2} + \frac{\beta^2}{8}\right)^{-1} T}
We present the unifying lemma for the estimation error bound under
\Cref{assum:main,assum:tail}:
\begin{lem}\label{lem:main}
  Suppose that \Cref{assum:main,assum:tail} hold with:
  \begin{subequations}\label{eqn:main-lem-assum}
    \begin{align}
      T &\geq 2\overline{t}_c,\label{eqn:main-lem-assum-T}\\
      k &= \max\left(\overline{t}_c,\left\lceil\log_{(1/\alpha)}\left(\frac{16\sigma^2 L^2 }{\beta^2 (1-\alpha)^2}\right)\right\rceil,
      \left\lceil\log_{(1/\alpha)}T\right\rceil \right) + 1,\label{eqn:main-lem-assum-k}\\
      \epsilon &\leq \epsbound.\label{eqn:main-lem-assum-eps}
    \end{align}
  \end{subequations}
  Then, with probability at least $1-\delta$, the estimator $\widehat{\Theta}$ in \cref{eqn:est} is unique and the following holds:
  \begin{align}\label{eqn:lem-main-res}
    \left\|\widehat{\Theta} - \Theta\right\| \leq 
    \frac{ 8 \sigma^2 L}{(1-\alpha) \beta^2}\left(\frac{\epsilon}{T} + \frac{1}{T^{1/2}}\right),
  \end{align}
  where $r,T,\Theta$ are defined in \cref{eqn:sys}, $L, \alpha, \beta, \sigma, \overline{t}_c$ are defined in \Cref{assum:main},
  and $k, \delta, \epsilon$ are defined in \Cref{assum:tail}. 
\end{lem}
\Cref{lem:main} suggests that the estimation error 
can be expressed in the form of $\bigO(\epsilon/T + T^{-1/2})$, 
and it can be rationalized as follows.
The constants in front are related to system stability, persistent excitation, and the bounded effects of noise.
As the system becomes less stable with $L \rightarrow \infty$ and $\alpha \rightarrow 1$,  
the error bound becomes looser. 
We emphasize the dichotomy between persistent excitation
and noise covariance.
When the system is more excitable with increasing $\beta$ and 
data becomes more informative, the error bound becomes tighter;
but when the noise covariance upper bound $\sigma^2$ increases 
and data is obscured by white noise,
the error becomes larger.
Our result depends on the empirical covariance 
tail bound $\epsilon$ defined in \Cref{assum:tail}, 
which can be derived from distribution-specific 
tail properties of the noise. 
Therefore, $\epsilon/T$ can be interpreted 
as the inverse of a signal-to-noise ratio,
dictating the error of the parameter estimate.
The sample complexity analysis reduces to finding the explicit form of $\epsilon$ 
in terms of $T$, $r$, $k$, $\delta$, 
and distribution-dependent variables for different classes of noise distributions,
which we will discuss in the next section.

The proof of \Cref{lem:main} is given in \Cref{sec:main-thm-proof}. 
A sketch is provided here. 
The estimation error of the \gls*{ols} estimator in \cref{eqn:est} 
can be decomposed into two terms:
\begin{align}\label{eqn:ls-decomp}
  \left\|\widehat{\Theta} - \Theta\right\| \leq \left\|\overline\Xi Z^\top \right\| \left\|(Z Z^\top)^{-1}\right\|,
\end{align}
where $Z = [z_0\; z_1\; \cdots\; z_{T-1}]$ and $\overline\Xi = [\Psi\xi_1\; \Psi\xi_2\; \cdots\; \Psi\xi_{T}]$.
We analyze the upper bound on the norms of $\overline\Xi Z^\top $ and $(Z Z^\top)^{-1}$.
For upper bounding $(ZZ^\top)^{-1}$,
the key challenge is the temporal correlation of the states $\{z_t\}_{t=0}^{T-1}$.
However, the correlation becomes exponentially weak due to (exponential) stability 
in \Cref{assum:stab}.
In other words, the stochastic process $\{z_t\}_{t=0}^{T-1}$ mixes exponentially fast.
We exploit this property to decompose $ZZ^\top$ in terms of the order of augmentation $k$,
which is bounded below by the persistent excitation condition in \Cref{assum:ctrb-sigma},
and the magnitude of the empirical covariance tail decay is controlled by \Cref{assum:tail}.
Upper bounding the norm of $\overline\Xi Z^\top$ follows similar strategies. 
The result in \Cref{lem:main} is then obtained 
by combining the upper bounds for $(ZZ^\top)^{-1}$ 
and $\overline\Xi Z^\top$ in \cref{eqn:ls-decomp}.

\section{Sample Complexity Bounds for Distributions with Bounded $p$th Moments}\label{sec:examples}
\label{sec:bound-p}
We now discuss several distribution-specific sufficient conditions for \Cref{assum:tail}.
We first describe the $k$ blocking strategy that decouples temporal correlation
of the augmented noise vectors $\{\bxi_{t:t-k+1}\}_{t=1}^{T}$ defined in \Cref{assum:tail}. 
This decomposition allows the application of distribution-specific concentration inequalities
to derive $\epsilon$ and $\delta$ explicitly
and establish the sample complexity bound of 
heavy-tailed distributions with bounded $p$th moment 
for $p>2$ in \Cref{thm:pthmom}.
We demonstrate our technique on 
sub-Gaussian distributions in \Cref{sec:thm-subgauss}
and sub-exponential distributions in \Cref{sec:thm-subexp}
for completeness.
Our estimation error bounds will depend polynomially on the system-specific constants
introduced in \Cref{assum:main}:
\begin{align}\label{eqn:cal-S}
  \mathcal{S}:=\{L, \sigma, 1/(1-\alpha), \log^{-1}(1/\alpha), 1/\beta, \overline{t}_c\},
\end{align}
along with a few more distribution-dependent terms to be introduced later.

\paragraph{Sketch of deriving empirical covariance tail bound}
We first decompose the term in \cref{eqn:tail} 
to a more convenient form 
for applying the subsequent concentration results.
Recall the \emph{augmented noise vectors} $\{\bxi_{t:t-k+1}\}_{t=1}^T$ defined in \cref{eqn:tail}.
The main technical challenge is the temporal correlation of the empirical covariance matrix of
the augmented noise vectors, 
as they share overlapping noise terms.

\begin{figure}[t]
  \centering
  \begin{tikzpicture}[font=\small]
    \matrix (mat) [matrix of math nodes, left delimiter={(}, right delimiter={)}, column sep=.4cm, row sep=.1cm]
    {
      \xi_1 & \xi_2 & \xi_3 & \xi_4 & \xi_5 & \xi_6 & \cdots & \xi_{T-2} & \xi_{T-1} & \xi_{T}   \\
      \xi_0 & \xi_1 & \xi_2 & \xi_3 & \xi_4 & \xi_5 & \cdots & \xi_{T-3} & \xi_{T-2} & \xi_{T-1} \\
      0     & \xi_0 & \xi_1 & \xi_2 & \xi_3 & \xi_4 & \cdots & \xi_{T-4} & \xi_{T-3} & \xi_{T-2} \\
    };
    \node[anchor=east,xshift=-.5cm] (eq) at (mat.west) {$\boldsymbol{\Xi}_k=$};
    \begin{scope}[on background layer]
      \draw[thick,gray!10!white,line width=25pt, line cap=round] (mat-1-1.center) -- (mat-3-1.center) node[below,overlay,gray] {$\bxi_{1:-1}$};      \draw[thick,gray!10!white,line width=25pt, line cap=round] (mat-1-2.center) -- (mat-3-2.center) node[below,overlay,gray] {$\bxi_{2:0}$};
      \draw[thick,gray!10!white,line width=25pt, line cap=round] (mat-1-3.center) -- (mat-3-3.center) node[below,overlay,gray] {$\bxi_{3:1}$};
      \draw[thick,blue!10!white,line width=25pt, line cap=round] (mat-1-4.center) -- (mat-3-4.center) node[below,overlay,blue] {$\bxi_{4:2}$};
      \draw[thick,green!10!white,line width=25pt, line cap=round] (mat-1-5.center) -- (mat-3-5.center) node[below,overlay,green!80!black] {$\bxi_{5:3}$};
      \draw[thick,red!10!white,line width=25pt, line cap=round] (mat-1-6.center) -- (mat-3-6.center) node[below,overlay,red] {$\bxi_{6:4}$};
      \draw[thick,blue!10!white,line width=25pt, line cap=round] (mat-1-8.center) -- (mat-3-8.center) node[below,overlay,blue] {$\bxi_{T-2:T-4}$};
      \draw[thick,green!10!white,line width=25pt, line cap=round] (mat-1-9.center) -- (mat-3-9.center) node[below,overlay,green!80!black] {$\bxi_{T-1:T-3}$};
      \draw[thick,red!10!white,line width=25pt, line cap=round] (mat-1-10.center) -- (mat-3-10.center) node[below,overlay,red] {$\bxi_{T:T-2}$};
    \end{scope}
  \end{tikzpicture}
  \vspace{.1in}
  \caption{Intuition behind empirical covariance decomposition using $k$ blocking when $k = 3$.}
  \label{fig:tail}
\end{figure}

We address this by dividing the augmented noise vectors $\bxi_\tau^{(k,i)}$ into $k$ groups 
such that $\bxi_\tau^{(k,i)}$ in each group are independent.
\Cref{fig:tail} shows the intuition, where $\bxi_\tau^{(k,i)}$ of the same color
are in one group.
The first $k$ augmented noise vectors are treated as a special case 
because some components are zero-padded vectors, 
so they do not satisfy the identical distribution condition.
Writing algebraically, we can express \cref{eqn:tail} in the form of grouped summations as follows:

\newcommand{\termonemono}{\sum_{\tau=1}^{T^{(k,i)}} \bxi^{(k,i)}_\tau ( \bxi^{(k,i)}_\tau)^\top-\mathbb{E}\left[\sum_{\tau=1}^{T^{(k,i)}} \bxi^{(k,i)}_\tau ( \bxi^{(k,i)}_\tau)^\top\right]}
\newcommand{\termtwomono}{\bxi^{(k,i)}_0 ( \bxi^{(k,i)}_0)^\top-\mathbb{E}\left[ \bxi^{(k,i)}_0 ( \bxi^{(k,i)}_0)^\top\right]}
\newcommand{\termonemonounder}{\sum_{\tau=1}^{T^{(k,i)}} \underline{\bxi}^{(k,i)}_\tau ( \underline{\bxi}^{(k,i)}_\tau)^\top-\mathbb{E}\left[\sum_{\tau=1}^{T^{(k,i)}} \underline{\bxi}^{(k,i)}_\tau ( \underline{\bxi}^{(k,i)}_\tau)^\top\right]}
\newcommand{\termtwomonounder}{\underline{\bxi}^{(k,i)}_0 ( \underline{\bxi}^{(k,i)}_0)^\top-\mathbb{E}\left[ \underline{\bxi}^{(k,i)}_0 ( \underline{\bxi}^{(k,i)}_0)^\top\right]}
\begin{align}\label{eqn:emp-cov-decomp}
  &\tail 
  \\ &=\underbrace{\sum_{i=0}^{k-1} \left(\underbrace{\termonemono}_{\Delta_{1,k}}\right)}_{\Delta_1} +\underbrace{\sum_{i=0}^{k-1} \left(\underbrace{\termtwomono}_{\Delta_{2,k}}\right)}_{\Delta_2}, \nonumber
\end{align}
where $\bxi^{(k,i)}_\tau= \bxi_{\tau k + i + 1 : \tau k + i + 2 - k}$ and  $T^{(k,i)}= \left\lfloor\frac{T-i-1}{k}\right\rfloor $.

We note that in \cref{eqn:emp-cov-decomp},
$\Delta_1$ represents the majority of the $T$ samples
and contains $k$ groups of summation, where $\Delta_{1,k}$ is composed of
the summation of independent \glspl*{rv}.
$\Delta_2$ contains the ``left-over'' terms with summation of $k$ terms, 
and each $\Delta_{2,k}$ is a single random matrix.
This decomposition allows us to apply the standard concentration inequalities 
for the empirical covariance of the \gls*{iid} augmented noise vectors $\bxi^{(k,i)}_\tau$ to bound $\Delta_1$
and that for a single vector $\bxi^{(k,i)}_0$ to bound $\Delta_2$.
Finally, we can combine these empirical covariance tail bounds 
to establish the bound in \cref{eqn:tail}.
This is the primary approach used for deriving $\epsilon$ 
of the noise distributions considered in the following result.

\paragraph{Main Result}
We establish the main result of this work,
sample complexity bound for distributions with bounded $p$th moment for $p>2$,
which covers common heavy-tail distributions such as Student's t, Pareto, and log-normal distributions.
Our main result \Cref{thm:pthmom} shows that as the number of bounded moments becomes limited,
the sample complexity upper bound degrades by $T^{1/p}$.

\newcommand{\pthmomepsilonone}{
  \left(\frac{2^{1+p}\kappa k^{p+1}}{\delta}\right)^{1/p}\log^{1/2}\left(\frac{8rk^2}{\delta}\right)r^{1/2}T^{1/2+1/p}
}
\newcommand{\pthmomepsilontwo}{
  \left(\frac{2^{1+p}\kappa  k^{p+1}}{\delta}\right)^{2/p}\log\left(\frac{8rk^2}{\delta}\right) rT^{2/p}
}
\newcommand{\pthmomepsilon}{
  \pthmomepsilonone + \pthmomepsilontwo
}
\newcommand{\pthmomboundone}{
  \left(8k\left(\frac{2\kappa r^{p/2} k^{p/2+1}}{\delta}\right)^{2/p}\log\left(\frac{8rk^2}{\delta}\right)\right)^{p/(p-2)}
}
\newcommand{\pthmomboundtwo}{
  \left(4\left(\frac{2\kappa r^{p/2} k^{p/2+1}}{\delta}\right)^{2/p}\log\left(\frac{8rk^2}{\delta}\right)\right)^{-p/2} 
}
\newcommand{\pthmomthmboundone}{
  \left(\frac{16r^{1/2}}{\beta^2}\left(\frac{2^{1+p}\kappa k^{p+1}}{\delta}\right)^{1/p}\log^{1/2}\left(\frac{8rk^2}{\delta}\right)\left(\frac{\sigma^2 L^2}{1-\alpha^2} + \frac{\beta^2}{8}\right)\right)^{1/(1/2-1/p)}
}
\newcommand{\pthmomthmboundtwo}{
  \left(\frac{16r}{\beta^2}\left(\frac{2^{1+p}\kappa  k^{p+1}}{\delta}\right)^{2/p}\log\left(\frac{8rk^2}{\delta}\right) \left(\frac{\sigma^2 L^2}{1-\alpha^2} + \frac{\beta^2}{8}\right)\right)^{1/(1-2/p)}
}
\newcommand{\pthmomfinalbound}{
  \frac{ 8 \sigma^2 L}{ (1-\alpha) \beta^2}\left(
    1+\left(\frac{2^{1+p}\kappa k^{p+1}}{\delta}\right)^{1/p}\log^{1/2}\left(\frac{8rk^2}{\delta}\right)
  \right)^2
  r^{1/2}T^{-1/2+1/p}
  + \frac{8\sigma^2 L}{ (1-\alpha) \beta^2 } T^{-1/2}
}

\begin{thm}[Sample Complexity for Bounded $p$th Moment]\label{thm:pthmom}
  Suppose that \Cref{assum:main} holds, and for all $t=0,1,\cdots,T$,
  \begin{align}\label{eqn:pthmom-assum-moment}
    \mathbb{E}[\|\xi_t\|^p] \leq \kappa r^{p/2}
  \end{align}
  with fixed $p$ such that $p>2$ and $\kappa \in\mathbb{R}_{>0}$.
  There exist $p$-dependent polynomials $\polyPburn(\cdot)$ and $\polyPbound(\cdot)$ such that for given $\delta\in(0,1)$, 
  if
  \begin{align} 
    \begin{aligned}[t]
      T \geq \polyPburn \left(\mathcal{S} \cup \left\{ r, 1/\delta, \kappa \right\}\right)
    \end{aligned}
    \label{eqn:pthmom-thm-assum-poly}
  \end{align}
  then the following holds with probability at least $1-\delta$:
  \begin{align}
    &\left\|\widehat{\Theta} - \Theta\right\| \leq 
    \polyPbound\left(\mathcal{S} \cup \left\{ \log(T), \log(r), 1/\delta, \kappa\right\}\right)r^{1/2}T^{-1/2+1/p},
    \label{eqn:pthmom-thm}
  \end{align}
  where
  $r,T,\Theta$ are defined in \cref{eqn:sys},
  $\widehat{\Theta}$ is defined in \cref{eqn:est},
  $L, \alpha, \beta, \sigma, \overline{t}_c$ are defined in \Cref{assum:main},
  and $\mathcal{S}$ is defined in \cref{eqn:cal-S}.

\end{thm}

The proof is detailed in \Cref{sec:thm-pthmom}. 
We first find an explicit form of $\epsilon$ using \Cref{lem:pthmom}
by deriving the empirical covariance tail bound of the augmented noise vectors 
in \cref{eqn:emp-cov-decomp} using the Markov inequality
and bounded $p$th moment property in \Cref{prop:pthmom-indcov}. 
For sufficiently large $T$,
we can apply the general concentration inequality 
\cite[Corollary 6.20]{wainwrightHighDimensionalStatisticsNonAsymptotic2019} 
for bounded distributions
using a truncation technique on the augmented noise vectors.
Lastly, we apply \Cref{lem:main} to find the final sample complexity bound. 

Our result shows that 
for fixed $p$, 
the estimation error is $\widetilde{\bigO}(r^{1/2}T^{-1/2+1/p})$ with probability at least $1-\delta$ 
for sufficiently large $T$ satisfying the conditions in \cref{eqn:pthmom-thm-assum-poly}.
Our sample complexity bound 
retains the dimension scale $r^{1/2}$ similar to the sub-Gaussian distributions in \Cref{thm:subgauss}
and sub-exponential distributions in \Cref{thm:subexp}.
Additionally, the bound is expressed in terms of the noise dimension $r$,
not the system dimension $d$, whereas previous works
have established the bound with respect to $d$. 
Thus, our result can be more advantageous for certain rank-deficient systems where
$r < d$. See \Cref{sec:arx} on the ARX model for more discussion.

Different from the sample complexity bounds recovered for light-tailed distributions, 
the heavy-tailed result deteriorates by a factor of $T^{1/p}$,
and the error depends on $1/\delta$ rather than $\log(1/\delta)$.
More precisely, the error depends on $(1/\delta)^{1/p}\log(1/\delta)$,
which converges to $\log(1/\delta)$ when all moments are finite.
The estimation error also increases polynomially with $\kappa$, 
the upper bound of the $p$th moment.
When $p \rightarrow \infty$, 
the error bound matches with that of light-tailed distributions,
which is expected as they have infinitely many bounded moments. 
As $p \rightarrow 2$, 
the estimation deteriorates by $T^{1/p}$,
indicating slower convergence when the number of bounded moments is limited
and the tail of the noise distribution becomes heavier.

\begin{remark}
Our sample complexity bounds only address exponentially stable systems
and do not apply to marginally stable systems,
which can be observed from the prefactors of $\polyPbound(\cdot)$:
$1/(1-\alpha)$ that diverge as $\alpha \rightarrow 1$.
Our approach relies on \Cref{lem:main}, 
which requires system stability 
to decouple the temporal noise correlation.
Previous works from \cite{Simchowitz2018-if,ziemannTutorialNonAsymptoticTheory2024}
have shown sample complexity bounds for marginally stable systems 
with sub-Gaussian noise.
Specifically, \cite{Simchowitz2018-if} employs the small-ball method 
and \cite{ziemannTutorialNonAsymptoticTheory2024} 
uses a less stringent lower empirical covariance tail bound 
specific to sub-Gaussian distributions. 
\cite{ziemannTutorialNonAsymptoticTheory2024}'s result 
depends on $1/\lambda_{\min}(\mathbb{E}[z_tz_t^\top])$, 
where $\lambda_{\min}(\mathbb{E}[z_tz_t^\top])$ may grow linearly in $T$
and makes the error bound $\bigO(T^0)$
in the worst case. But it provides a more granular sample complexity bound 
when the above-mentioned minimum eigenvalue grows slower than linearly in $T$. 
\end{remark}
\begin{remark}
Numerical simulation results are provided in \Cref{sec:numerical} 
to demonstrate the sample complexity bounds for heavy-tailed noise distributions with bounded $p$th moment
using the \gls*{arx} model.
The simulations do not show that heavy-tailed distributions,
such as the Student's t and Pareto distributions,
deviate from $T^{-1/2}$ when $p$ is close to $2$,
which suggests that the error bound 
is tighter than what our results state.
Thus, investigating the optimality 
of the $T^{1/p}$ dependence and a matching lower bound
in the estimation error bound
is an important future direction.
\end{remark}

\section{Special Case: ARX Model}\label{sec:arx}
We now consider the \gls*{arx} model 
as a special case to our setting in \Cref{sec:setting}
and describe the sufficient conditions such that \Cref{assum:main} is satisfied
for the ARX model.
Consider the following \gls*{arx} model with randomly sampled inputs:
\begin{align}\label{eqn:arx}
  y_t = \sum_{i=1}^{q} \left(A_i y_{t-i} + B_i u_{t-i}\right) + Fw_t,\quad u_t = Gv_t,
\end{align}
where $u_t\in\mathbb{R}^{d_u}$ is the input, 
$y_t\in\mathbb{R}^{d_y}$ is the output,
$v_t$ is an \gls*{rv} taking values in $\mathbb{R}^{d_u}$,
$w_t$ is the random noise taking values in $\mathbb{R}^{d_y}$, 
$A_i\in\mathbb{R}^{d_y\times d_y}$, $B_i\in\mathbb{R}^{d_y\times d_u}$, 
$F\in\mathbb{R}^{d_y\times d_y}$, and $G\in\mathbb{R}^{d_u\times d_u}$
are unknown system parameters.
The random sequence $\{u_t\}_{t=-q+1}^T$ is injected 
as inputs shown in \cref{eqn:arx} for persistent excitation.
We observe that \cref{eqn:arx} can be rewritten in the form of 
\cref{eqn:sys},
recovering the structure of the \gls*{var} model.
This is explicitly shown in \cref{eqn:arx-augmented} of \Cref{sec:arx-appendix}.

The system parameters 
$\{A_i\}_{i=1}^q$ and $\{B_i\}_{i=1}^q$ are estimated
from the input-output data $\{(u_t,y_t)\}_{t=-q+1}^T$,
where $(u_t,y_t)$ are zero vectors for $t < 0$. 
The least-squares estimator can be defined as follows:
\begin{align}\label{eqn:arx-estimator}
  \{\widehat A_i,\widehat B_i\}_{i=1}^q\in
  \arg\min_{\{A_i,B_i\}_{i=1}^q}\sum_{t=1}^{T} \left\|y_t - \sum_{i=1}^{q} \left(A_i y_{t-i} + B_i u_{t-i}\right)\right\|^2.
\end{align}
Due to row-wise independence of the least-squares estimation, 
$\{\widehat{A}_i,\widehat{B}_i\}_{i=1}^q$ can be recovered from $\widehat{\Theta}$, 
which is estimated using \gls*{ols}.
Therefore, it suffices to analyze the performance of the estimator in \cref{eqn:est} 
applied to the ARX model,
where the sample complexity results established in \Cref{sec:unifying,sec:examples} 
can be directly used for \gls*{arx} model identification \cref{eqn:arx-estimator}. 
We now introduce such sufficient conditions.
\begin{assum} \label{assum:arx}
  Consider the \gls*{arx} model in \cref{eqn:arx}, written explicitly in the augmented form in \cref{eqn:arx-augmented}.
  \begin{enumerate}[label=(\alph*),leftmargin=*,ref=\theassum(\alph*)]
    \item \label[assum]{assum:arx-iid}
      We assume $\xi_t=(w_t,v_t)$ satisfies the conditions in \Cref{assum:iid}.
    \item \label[assum]{assum:arx-main}
      For given $c_A,c_B,c_C,c_D \in \mathbb{R}_{>0}$ and $\gamma\in(0,1)$, we assume the following:
      \begin{align*}
        \|\bA^t\| &\leq c_A \gamma^t,\quad \forall t=0,1,\cdots,\quad \|\bB \|  \leq c_B,\quad \|\bC \|  \leq c_C,\quad \|\bD \|  \leq c_D,
      \end{align*}
      where 
      \begin{align}
        \begin{aligned}[t]
          \bA &=
                \begin{bmatrix}
                  A_1 & \cdots & A_{q-1} & A_q \\
                  I   & \cdots & 0       & 0   \\
                  \vdots   & \ddots  & \vdots & \vdots \\
                  0   & \cdots & I       & 0  \\
                \end{bmatrix}\\ 
          \bC &=
                \begin{bmatrix}
                    I & -A_1 & \cdots & -A_{q-1} \\
                    0 &  \ddots    & \ddots & \vdots \\
                    \vdots &  \ddots    &  \ddots      & -A_1 \\
                    0  & \cdots     &   0     & I
                \end{bmatrix}\quad
        \end{aligned}
        \begin{aligned}[t]
          \bB &=
                \begin{bmatrix}
                  B_1 & \cdots & B_{q-1} & B_q \\
                  0   & \cdots & \cdots       & 0   \\
                  \vdots   & \ddots  & \ddots & \vdots \\
                  0   & \cdots & \cdots       & 0  \\
                \end{bmatrix}\\
          \bD &=
                \begin{bmatrix}
                    0 & -B_1 & \cdots & -B_{q-1} \\
                    0 &  \ddots    & \ddots & \vdots \\
                    \vdots &  \ddots    &  \ddots      & -B_1 \\
                    0  & \cdots     &   0     & 0
                \end{bmatrix}.
        \end{aligned} \label{eqn:arx-ABCD-def}
      \end{align}
    \item \label[assum]{assum:arx-noise}
      $F$ and $G$ are square matrices, and the following holds for given $\overline{\sigma}_F,\underline{\sigma}_F,\overline{\sigma}_G,\underline{\sigma}_G>0$:
      \begin{align*}
        \underline{\sigma}^2_F I \preceq FF^\top &\preceq \overline{\sigma}^2_F I,\quad \underline{\sigma}^2_G I \preceq GG^\top \preceq \overline{\sigma}^2_G I.
      \end{align*}
  \end{enumerate}
\end{assum}

\Cref{assum:arx-iid} restates \Cref{assum:iid} for the noise term $\xi_t$.
\Cref{assum:arx-main} guarantees the stability of the submatrix $\bA$ and the boundedness of $\bB$, 
which is sufficient to ensure the stability of the augmented system. 
The bounds on $\bC$ and $\bD$ are introduced 
to show the persistent excitation property of the covariates.
\Cref{assum:arx-noise} assumes the noise injection matrices $F$ and $G$ are bounded and 
thus invertible. 
The upper bound condition is related to the bounded effect of noise condition in \Cref{assum:bounded-cov}, 
while the lower bound condition is related to the persistent excitation condition in \Cref{assum:ctrb-sigma}. 
  
We are now ready to state the specialization of our main results on the \gls*{arx} model.
\begin{thm}\label{thm:arx}
  Under \Cref{assum:arx}, the conditions in \Cref{assum:main} are satisfied with
  \begin{align}\label{eqn:special-arx-params}
    \begin{aligned}[t]
      d &= q(d_y + d_u),\quad r =d_y + d_u, \quad L =  \left(\max(c_A, \gamma^{-q}) + c_A c_B \gamma^{-1}\frac{\gamma^{-q}-1}{\gamma^{-1}-1} \right),\;\;\alpha = \gamma,\\
      \beta &= \frac{\min(\underline{\sigma}_F,\underline{\sigma}_G)}{\max(1, c_C) + c_D},\;\; \sigma = \max(\overline\sigma_F,\overline\sigma_G), \quad \overline{t}_c = q,
    \end{aligned}
      \end{align}
  where $q, d_y, d_u$ are defined in \cref{eqn:arx}, $c_A, c_B, c_C, c_D, \gamma, \underline{\sigma}_F, \overline\sigma_F, \underline{\sigma}_G, \overline\sigma_G$
  are defined in \Cref{assum:arx}.
  Thus, \Cref{thm:subgauss,thm:subexp,thm:pthmom} hold with \Cref{assum:main} replaced by \Cref{assum:arx} with $d, r, L,\alpha,\beta,\sigma,\overline{t}_c $ in \cref{eqn:special-arx-params}. 
\end{thm}
The proof is detailed in \Cref{sec:thm-arx}.
Our result shows that the error bound is 
$\widetilde{\bigO}((d_y+d_u)^{1/2}T^{-1/2+1/p})$ for bounded $p$th moment noise with $p>2$, 
and $\widetilde{\bigO}((d_y+d_u)^{1/2}T^{-1/2})$ 
for sub-Gaussian and sub-exponential noise,
with high probability and sufficiently large $T$.

\begin{remark}
  Compared to the existing error bounds for \gls*{arx} model identification in the literature, 
  our result in \Cref{thm:arx} has the following advantage: 
  \Cref{thm:arx} does not explicitly depend on the order of the \gls*{arx} model $q$. 
  Thus, when the model order is large,
   our result enjoys a tighter sample complexity bound than the existing results that depend on $q$ explicitly 
   \citep{ziemannTutorialNonAsymptoticTheory2024}. 
   This improvement was made by decoupling the system dimension $d$ and the noise dimension $r$
   in \cref{eqn:sys}, 
   and by expressing the sample complexity in terms of $r$ instead of $d$.
   For the \gls*{arx} model in \cref{eqn:arx}, the noise dimension $r=d_y+d_u$ 
   is substantially smaller than the system dimension $d=q(d_y+d_u)$ 
   when the model order $q$ is large.
   Note that system-dependent coefficients
   can be dependent of $q$. Specifically, the $q$ dependence appears
   in the stability coefficient $L$ due to explicit bounding of $A$,
   which can be removed if expressed in terms of the system matrices.
   When $T$ becomes large, specifically $T\geq \alpha^{-\overline{t}_c}$, 
   the dependence on $\overline{t}_c$ becomes negligible,
   as expressed in \cref{eqn:main-lem-assum-k}.
\end{remark}

\section{Conclusion}\label{sec:conclusion}
This work analyzed non-asymptotic sample complexity 
of linear system identification concerning exponentially stable systems
subjected to heavy-tailed noise distributions 
with bounded $p$th moment for $p>2$. 
We provide a unifying lemma applicable to general noise distributions,
and we further demonstrate our approach to sub-Gaussian and sub-exponential distributions,
establishing competitive sample complexity bounds for each noise distribution class.
Our results reveal that sample complexity 
depends on the controllability properties of the system and the tail behavior of the noise distribution.
Specifically, we found that for systems affected by heavy-tailed distributions, 
the sample complexity deteriorates with decreasing number of bounded moments $p$.
Lastly, we specialized our results to the \gls*{arx} model, 
where the model dimension of the parameter error does not explicitly depend on the autoregressive order,
thus demonstrating that our framework can improve upon existing results in literature.

There remain open questions from our work. 
Our technique relies on a mixing time argument
due to exponential stability assumption,
which can be avoided for \gls*{ols} of 
light-tailed noise \citep{Simchowitz2018-if,ziemannTutorialNonAsymptoticTheory2024}.
Can this also be avoided for heavy-tailed noise?
Additionally, the $T^{1/p}$ dependence of heavy-tailed distribution 
suggested by our theoretical results 
is not corroborated by the numerical results.
This suggests that our approach,
applying a truncation technique on a Markov inequality, may be loose, 
and it is possible that a more refined analysis can yield a tighter bound.
Future work includes deriving the lower sample complexity bound,
extending our analysis to marginally stable systems, and 
addressing heavy-tailed noise distributions with bounded $p$th moment for $p \leq 2$.

\begin{ack}
This material is based upon work supported by the U.S. Department of
Energy, Office of Science, Office of Advanced Scientific Computing Research, Department of
Energy Computational Science Graduate Fellowship under Award Number DE-SC0025528.

Disclaimer: This report was prepared as an account of work sponsored by an agency of the
United States Government. Neither the United States Government nor any agency thereof, nor
any of their employees, makes any warranty, express or implied, or assumes any legal liability
or responsibility for the accuracy, completeness, or usefulness of any information, apparatus,
product, or process disclosed, or represents that its use would not infringe privately owned
rights. Reference herein to any specific commercial product, process, or service by trade name,
trademark, manufacturer, or otherwise does not necessarily constitute or imply its
endorsement, recommendation, or favoring by the United States Government or any agency
thereof. The views and opinions of authors expressed herein do not necessarily state or reflect
those of the United States Government or any agency thereof.
\end{ack}

\bibliography{main}

@misc{ziemannTutorialNonAsymptoticTheory2024,
  title = {A {{Tutorial}} on the {{Non-Asymptotic Theory}} of {{System Identification}}},
  author = {Ziemann, Ingvar and Tsiamis, Anastasios and Lee, Bruce and Jedra, Yassir and Matni, Nikolai and Pappas, George J.},
  year = {2024},
  month = jun,
  number = {arXiv:2309.03873},
  eprint = {2309.03873},
  primaryclass = {eess},
  publisher = {arXiv},
  doi = {10.48550/arXiv.2309.03873},
  urldate = {2024-12-31},
  archiveprefix = {arXiv}
}

@book{wainwrightHighDimensionalStatisticsNonAsymptotic2019,
  title = {High-{{Dimensional Statistics}}: {{A Non-Asymptotic Viewpoint}}},
  shorttitle = {High-{{Dimensional Statistics}}},
  author = {Wainwright, Martin J.},
  year = {2019},
  series = {Cambridge {{Series}} in {{Statistical}} and {{Probabilistic Mathematics}}},
  publisher = {Cambridge University Press},
  address = {Cambridge},
  doi = {10.1017/9781108627771},
  urldate = {2025-01-12},
  isbn = {978-1-108-49802-9}
}

@online{oymakNonasymptoticIdentificationLTI2019,
  title = {Non-Asymptotic {{Identification}} of {{LTI Systems}} from a {{Single Trajectory}}},
  author = {Oymak, Samet and Ozay, Necmiye},
  year = {2019},
  date = {2019-02-03},
  eprint = {1806.05722},
  eprinttype = {arXiv},
  eprintclass = {cs},
  doi = {10.48550/arXiv.1806.05722},
  url = {http://arxiv.org/abs/1806.05722},
  urldate = {2025-04-06},
  pubstate = {prepublished},
}

@ARTICLE{Van_Overschee1994-th,
  title   = "{N4SID}: subspace algorithms for the identification of combined
             deterministic- stochastic systems",
  author  = "Van Overschee, Peter and de Moor, Bart",
  journal = "No. I. pp.",
  volume  =  75,
  pages   =  93,
  year    =  1994
}

@ARTICLE{Qin2006-iz,
  title     = "An overview of subspace identification",
  author    = "Qin, S Joe",
  journal   = "Comput. Chem. Eng.",
  publisher = "Elsevier BV",
  volume    =  30,
  number    = "10-12",
  pages     = "1502--1513",
  month     =  "12~" # sep,
  year      =  2006,
  doi       = "10.1016/j.compchemeng.2006.05.045",
  language  = "en"
}

@ARTICLE{Tsiamis2022-pb,
  title         = "Statistical learning theory for control: A finite sample
                   perspective",
  author        = "Tsiamis, Anastasios and Ziemann, Ingvar and Matni, Nikolai
                   and Pappas, George J",
  journal       = "arXiv [eess.SY]",
  month         =  "12~" # sep,
  year          =  2022,
  archivePrefix = "arXiv",
  primaryClass  = "eess.SY"
}

@INPROCEEDINGS{Bakshi2023-az,
  title     = "A new approach to learning linear dynamical systems",
  author    = "Bakshi, Ainesh and Liu, Allen and Moitra, Ankur and Yau, Morris",
  booktitle = "Proceedings of the 55th Annual ACM Symposium on Theory of
               Computing",
  publisher = "ACM",
  address   = "New York, NY, USA",
  pages     = "335--348",
  month     =  "2~" # jun,
  year      =  2023,
  doi       = "10.1145/3564246.3585247",
  language  = "en"
}

@ARTICLE{Simchowitz2018-if,
  title         = "Learning without mixing: Towards A sharp analysis of linear
                   system identification",
  author        = "Simchowitz, Max and Mania, Horia and Tu, Stephen and Jordan,
                   Michael I and Recht, Benjamin",
  journal       = "arXiv [cs.LG]",
  month         =  "22~" # feb,
  year          =  2018,
  archivePrefix = "arXiv",
  primaryClass  = "cs.LG"
}

@ARTICLE{Rosolia2018-xx,
  title     = "Data-driven predictive control for autonomous systems",
  author    = "Rosolia, Ugo and Zhang, Xiaojing and Borrelli, Francesco",
  journal   = "Annu. Rev. Control Robot. Auton. Syst.",
  publisher = "Annual Reviews",
  volume    =  1,
  number    =  1,
  pages     = "259--286",
  month     =  "28~" # may,
  year      =  2018,
  doi       = "10.1146/annurev-control-060117-105215",
  language  = "en"
}

@ARTICLE{Fazel2018-hn,
  title     = "Global convergence of policy gradient methods for the linear
               quadratic regulator",
  author    = "Fazel, Maryam and Ge, Rong and Kakade, S and Mesbahi, M",
  editor    = "Dy, Jennifer and Krause, Andreas",
  journal   = "ICML",
  publisher = "PMLR",
  volume    =  80,
  pages     = "1466--1475",
  month     =  "15~" # jan,
  year      =  2018
}

@incollection{ljung1998system,
  title={System identification},
  author={Ljung, Lennart},
  booktitle={Signal analysis and prediction},
  pages={163--173},
  year={1998},
  publisher={Springer}
}

@article{hu2023toward,
  title={Toward a theoretical foundation of policy optimization for learning control policies},
  author={Hu, Bin and Zhang, Kaiqing and Li, Na and Mesbahi, Mehran and Fazel, Maryam and Ba{\c{s}}ar, Tamer},
  journal={Annual Review of Control, Robotics, and Autonomous Systems},
  volume={6},
  number={1},
  pages={123--158},
  year={2023},
  publisher={Annual Reviews}
}

@ARTICLE{Kanakeri2024-xh,
  title         = "Outlier-robust linear system identification under
                   heavy-tailed noise",
  author        = "Kanakeri, Vinay and Mitra, Aritra",
  journal       = "arXiv [eess.SY]",
  month         =  "31~" # dec,
  year          =  2024,
  archivePrefix = "arXiv",
  primaryClass  = "eess.SY"
}

@INPROCEEDINGS{Sattar2025-lr,
  title     = "Learning linear dynamics from bilinear observations",
  author    = "Sattar, Yahya and Jedra, Yassir and Dean, Sarah",
  booktitle = "2025 American Control Conference (ACC)",
  publisher = "IEEE",
  pages     = "3109--3115",
  month     =  "8~" # jul,
  year      =  2025,
  doi       = "10.23919/acc63710.2025.11107791",
  language  = "en"
}

@misc{dean2017samplecomplexitylinearquadratic,
      title={On the Sample Complexity of the Linear Quadratic Regulator}, 
      author={Sarah Dean and Horia Mania and Nikolai Matni and Benjamin Recht and Stephen Tu},
      year={2017},
      eprint={1710.01688},
      archivePrefix={arXiv},
      primaryClass={math.OC},
      url={https://arxiv.org/abs/1710.01688}, 
}

@BOOK{Rawlings2017-nw,
  title     = "Model predictive control: Theory, computation, and design",
  author    = "Rawlings, James Blake and Mayne, David Q and Diehl, Moritz",
  publisher = "No ll Publishing",
  address   = "Madison, Wisconsin",
  edition   =  2,
  year      =  2017,
  language  = "en",
  chapter   = {1},
}

@BOOK{Sutton2018-ji,
  title     = "Reinforcement Learning: An Introduction",
  author    = "Sutton, Richard S and Barto, Andrew G",
  publisher = "Bradford Books",
  address   = "Cambridge, MA",
  edition   =  2,
  series    = "Adaptive Computation and Machine Learning series",
  month     =  "13~" # nov,
  year      =  2018,
  language  = "en",
  chapter   = {1}
}

@BOOK{Nair2022-wg,
  title     = "The fundamentals of heavy tails: Properties, emergence, and
               estimation",
  author    = "Nair, Jayakrishnan and Wierman, Adam and Zwart, Bert",
  publisher = "Cambridge University Press",
  address   = "Cambridge, England",
  series    = "Cambridge Series in Statistical and Probabilistic Mathematics",
  month     =  "9~" # jun,
  year      =  2022,
  doi       = "10.1017/9781009053730",
  language  = "en",
  chapter   = {1}
}

@inproceedings{qin1997overview,
  title={An overview of industrial model predictive control technology},
  author={Qin, S Joe and Badgwell, Thomas A},
  booktitle={AIche symposium series},
  volume={93},
  number={316},
  pages={232--256},
  year={1997},
  organization={New York, NY: American Institute of Chemical Engineers, 1971-c2002.}
}

@BOOK{Van-Overschee1996-ee,
  title     = "Subspace identification for linear systems",
  author    = "Van Overschee, Peter and De Moor, Bart",
  publisher = "Springer US",
  address   = "Boston, MA",
  year      =  1996,
  doi       = "10.1007/978-1-4613-0465-4"
}

@ARTICLE{Faradonbeh2017-ju,
  title         = "Finite time identification in unstable linear systems",
  author        = "Faradonbeh, Mohamad Kazem Shirani and Tewari, Ambuj and
                   Michailidis, George",
  journal       = "arXiv [cs.SY]",
  month         =  "4~" # oct,
  year          =  2017,
  archivePrefix = "arXiv",
  primaryClass  = "cs.SY"
}

@ARTICLE{Yu1994-uj,
  title     = "Rates of convergence for empirical processes of stationary mixing
               sequences",
  author    = "Yu, Bin",
  journal   = "Ann. Probab.",
  publisher = "Institute of Mathematical Statistics",
  volume    =  22,
  number    =  1,
  pages     = "94--116",
  month     =  "1~" # jan,
  year      =  1994,
  doi       = "10.1214/aop/1176988849"
}
\bibliographystyle{plainnat}

\newpage
\appendix
\crefalias{section}{appendix}

\section{Proof of \Cref{lem:main}}\label{sec:main-thm-proof}

This section proves \Cref{lem:main}, 
which generalizes the upper bound of parameter estimation error
as expressed in \cref{eqn:ls-decomp} 
for any noise distribution. 
We first find the analytical solution 
to the parameter error in \Cref{sec:error-decompo}. 
We then decompose the empirical covariance
and define our notation for the proof in \Cref{sec:emp-cov-decomp}.
\Cref{subsec:low-bound-ZZT} finds the lower bound of $\|ZZ^\top\|$
and 
\Cref{subsec:upper-bound-XiZT} shows the upper bound of $\|\overline\Xi Z^\top\|$.
Lastly, the final bound is observed in \Cref{subsec:lemma-final-bound}.

\subsection{Decomposition of error terms}\label{sec:error-decompo}
We first consider a decomposition of the estimation error.
The least-squares formulation in \cref{eqn:est} can be rewritten as follows:
\begin{align}\label{eqn:est-2}
  \min_{\Theta}\left\|\overline{Z}-\Theta Z\right\|_F^2,
\end{align}
where $Z = [z_0\; z_1\; \cdots\; z_{T-1}]$ and $\overline{Z} = [z_1\; z_2\; \cdots\; z_{T}]$.
Temporarily assuming $Z Z^\top\succ 0$ (it will be shown in \Cref{subsec:low-bound-ZZT} that this condition holds
under \Cref{assum:main,assum:tail}),
the solutions to \cref{eqn:est-2} can be expressed as follows:
\begin{align} \label{eqn:theta-hat}
  \widehat{\Theta} = \overline{Z}Z^\top (Z Z^\top)^{-1}.
\end{align}
Furthermore, one can observe that
\begin{align}\label{eqn:matform}
  \begin{aligned}
    \overline{Z}
    = \Theta Z + \overline\Xi,
  \end{aligned}
\end{align}
where $\overline\Xi = [\Psi\xi_1\; \Psi\xi_2\; \cdots\; \Psi\xi_{T}]$.
By right-multiplying $Z^\top (Z Z^\top)^{-1}$ to both sides of \cref{eqn:matform}, we obtain:
\begin{align}\label{eqn:theta-diff}
\widehat{\Theta} - \Theta = \overline\Xi Z^\top (Z Z^\top)^{-1}.
\end{align}
The error decomposition in \cref{eqn:ls-decomp} follows from \cref{eqn:theta-diff}.

\subsection{Empirical covariance decomposition}\label{sec:emp-cov-decomp}
We first derive a decomposition of the empirical covariance matrix
\begin{align}\label{eqn:emp-cov}
  \begin{aligned}[t]
    \sum_{t=1}^T \bxi_{t:t-k+1}\bxi_{t:t-k+1}^\top
    &=
      \begin{bmatrix}
        \xi_1
        & \xi_2 &\cdots & \xi_{k-1}   &\cdots &\xi_{T} \\
        \xi_0& \xi_1 & \cdots & \xi_{k-2}&\cdots &\xi_{T-1}    \\
        0& \xi_0 & \cdots & \xi_{k-3}  &\cdots &\xi_{T-2}  \\
        \vdots&\vdots&\ddots&\vdots&\ddots&\vdots\\
        0& 0 & \cdots & \xi_{0}    &\cdots &\xi_{T-k+1}\\
      \end{bmatrix}
      \begin{bmatrix}
        \xi_1
        & \xi_2 &\cdots & \xi_{k-1}   &\cdots &\xi_{T} \\
        \xi_0& \xi_1 & \cdots & \xi_{k-2}&\cdots &\xi_{T-1}    \\
        0& \xi_0 & \cdots & \xi_{k-3}  &\cdots &\xi_{T-2}  \\
        \vdots&\vdots&\ddots&\vdots&\ddots&\vdots\\
        0& 0 & \cdots & \xi_{0}    &\cdots &\xi_{T-k+1}\\
      \end{bmatrix}^\top\\
    &=
      \begin{bmatrix}
        \Xi_{-1}\\
        \Xi_{0}\\
        \vdots\\
        \Xi_{k-2}
      \end{bmatrix}
      \begin{bmatrix}
        \Xi_{-1}\\
        \Xi_{0}\\
        \vdots\\
        \Xi_{k-2}\\
      \end{bmatrix}^\top\\
    &=
      \begin{bmatrix}
        \Xi_{-1}\Xi_{-1}^\top
        & \Xi_{-1}\Xi_{0}^\top&\cdots&\Xi_{-1}\Xi_{k-2}^\top  \\
        \Xi_{0}\Xi_{-1}^\top & \Xi_{0}\Xi_{0}^\top&\cdots&\Xi_{0}\Xi_{k-2}^\top  \\
        \vdots&\vdots&\ddots&\vdots\\
        \Xi_{k-2}\Xi_{-1}^\top &\Xi_{k-2}\Xi_{0}^\top &\cdots&\Xi_{k-2}\Xi_{k-2}^\top\\
      \end{bmatrix}\\
    &=
      \begin{bmatrix}
        \widehat{Q}_{-1,-1}
        & \widehat{Q}_{-1,0}&\cdots&\widehat{Q}_{-1,k-2}  \\
        \widehat{Q}_{0,-1} & \widehat{Q}_{0,0}&\cdots&\widehat{Q}_{0,k-2}  \\
        \vdots&\vdots&\ddots&\vdots\\
        \widehat{Q}_{k-2,-1} &\widehat{Q}_{k-2,0} &\cdots&\widehat{Q}_{k-2,k-2}\\
      \end{bmatrix}
  \end{aligned}
\end{align}
where
\begin{align}\label{eqn:Qij-defn}
  {{\Xi}}_t := \begin{bmatrix} \xi_{-t}&\xi_{-t+1}&\cdots&\xi_{-t+T-1}\\ \end{bmatrix},\quad
  Q_{t,t'} := \mathbb{E}[{\Xi}_t {\Xi}_{t'}^\top],\quad \widehat{Q}_{t,t'} := {\Xi}_t {\Xi}_{t'}^\top.
\end{align}
Similarly,
\begin{align}\label{eqn:emp-cov-exp}
  \mathbb{E}\left[\sum_{t=1}^T \bxi_{t:t-k+1}\bxi_{t:t-k+1}^\top\right] =
  \begin{bmatrix}
    {Q}_{-1,-1}
    & {Q}_{-1,0}&\cdots&{Q}_{-1,k-2}  \\
    {Q}_{0,-1} & {Q}_{0,0}&\cdots&{Q}_{0,k-2}  \\
    \vdots&\vdots&\ddots&\vdots\\
    {Q}_{k-2,-1} &{Q}_{k-2,0} &\cdots&{Q}_{k-2,k-2}\\
  \end{bmatrix}.
\end{align}

\subsection{Lower bound of $Z Z^\top$} \label{subsec:low-bound-ZZT}
We express the matrix $Z$ as follows:
\begin{align*}
  Z
   & =
  \begin{bmatrix}
    z_0 & z_1 & z_2 & \cdots & z_{T-1} \\
  \end{bmatrix}   \\
   & =
  \begin{bmatrix}
    \Psi & \Theta \Psi& \Theta^2 \Psi& \cdots & \Theta^{T-1}\Psi
  \end{bmatrix}
  \begin{bmatrix}
    \xi_0  & \xi_1  & \cdots & \xi_{T-1} \\
    0      & \xi_0  & \cdots & \xi_{T-2} \\
    \vdots & \vdots & \ddots & \vdots    \\
    0      & 0      & \cdots & \xi_{0}   \\
  \end{bmatrix} \\
   & =
  \begin{bmatrix}
    \Psi & \Theta \Psi& \Theta^2 \Psi& \cdots & \Theta^{T-1}\Psi
  \end{bmatrix}
  \begin{bmatrix}
    \Xi_0  \\
    \Xi_1  \\
    \vdots \\
    \Xi_{T-1}
  \end{bmatrix}.
\end{align*}

We now decompose the matrix $Z Z^\top$ as follows:
\begin{align}\label{eqn:ZZtdecomp}
  Z Z^\top
  = &
  \begin{aligned}[t]
     & \begin{bmatrix}
         \Psi & \Theta \Psi & \Theta^2 \Psi & \cdots & \Theta^{k-2} \Psi
       \end{bmatrix}
    \begin{bmatrix}
      \Xi_0  \\
      \Xi_1  \\
      \vdots \\
      \Xi_{k-2}
    \end{bmatrix}
    \begin{bmatrix}
      \Xi_0  \\
      \Xi_1  \\
      \vdots \\
      \Xi_{k-2}
    \end{bmatrix}^\top
    \begin{bmatrix}
      \Psi & \Theta \Psi & \Theta^2 \Psi & \cdots & \Theta^{k-2} \Psi
    \end{bmatrix}^\top    \\
     & +
    \begin{bmatrix}
      \Theta^{k-1} \Psi& \Theta^k \Psi& \cdots & \Theta^{T-1}\Psi
    \end{bmatrix}
    \begin{bmatrix}
      \Xi_{k-1}   \\
      \Xi_k \\
      \vdots    \\
      \Xi_{T-1}
    \end{bmatrix}
    \begin{bmatrix}
      \Xi_{0} \\
      \Xi_{1} \\
      \vdots  \\
      \Xi_{k-2}
    \end{bmatrix}^\top
    \begin{bmatrix}
      \Psi & \Theta \Psi & \Theta^2 \Psi & \cdots & \Theta^{k-2} \Psi
    \end{bmatrix}^\top    \\
     & +
    \begin{bmatrix}
      \Psi & \Theta \Psi & \Theta^2 \Psi & \cdots & \Theta^{k-2} \Psi
    \end{bmatrix}
    \begin{bmatrix}
      \Xi_0  \\
      \Xi_1  \\
      \vdots \\
      \Xi_{k-2}
    \end{bmatrix}
    \begin{bmatrix}
      \Xi_{k-1}   \\
      \Xi_k \\
      \vdots    \\
      \Xi_{T-1}
    \end{bmatrix}^\top
    \begin{bmatrix}
      \Theta^{k-1} \Psi& \Theta^k \Psi& \cdots & \Theta^{T-1}\Psi
    \end{bmatrix}^\top  \\
     & +
    \begin{bmatrix}
      \Theta^{k-1} \Psi& \Theta^k \Psi& \cdots & \Theta^{T-1}\Psi
    \end{bmatrix}
    \begin{bmatrix}
      \Xi_{k-1}   \\
      \Xi_k \\
      \vdots    \\
      \Xi_{T-1}
    \end{bmatrix}
    \begin{bmatrix}
      \Xi_{k-1}   \\
      \Xi_k \\
      \vdots    \\
      \Xi_{T-1}
    \end{bmatrix}^\top
    \begin{bmatrix}
      \Theta^{k-1} \Psi& \Theta^k \Psi& \cdots & \Theta^{T-1}\Psi
    \end{bmatrix}^\top.
  \end{aligned}
\end{align}
We can further decompose the first term in \cref{eqn:ZZtdecomp}
in the following way:
\begin{subequations}\label{eqn:decomp-2}
  \begin{align}
     & \begin{bmatrix}
         \Psi & \Theta \Psi & \Theta^2 \Psi & \cdots & \Theta^{k-2} \Psi
       \end{bmatrix}
    \begin{bmatrix}
      \Xi_0  \\
      \Xi_1  \\
      \vdots \\
      \Xi_{k-2}
    \end{bmatrix}
    \begin{bmatrix}
      \Xi_0  \\
      \Xi_1  \\
      \vdots \\
      \Xi_{k-2}
    \end{bmatrix}^\top
    \begin{bmatrix}
      \Psi & \Theta \Psi & \Theta^2 \Psi & \cdots & \Theta^{k-2} \Psi
    \end{bmatrix}^\top           \\
     & \stackrel{\mathclap{\cref{eqn:emp-cov}}}{=}\begin{bmatrix}
          \Psi & \Theta \Psi & \Theta^2 \Psi & \cdots & \Theta^{k-2} \Psi
        \end{bmatrix}
    \begin{bmatrix}
      \widehat{Q}_{0,0}   & \cdots & \widehat{Q}_{0,k-2}   \\
      \vdots              &        & \vdots                \\
      \widehat{Q}_{k-2,0} & \cdots & \widehat{Q}_{k-2,k-2}
    \end{bmatrix}
    \begin{bmatrix}
      \Psi & \Theta \Psi & \Theta^2 \Psi & \cdots & \Theta^{k-2} \Psi
    \end{bmatrix}^\top           \\
     & \stackrel{\mathclap{\cref{eqn:emp-cov-exp}}}{=}\begin{bmatrix}
          \Psi & \Theta \Psi & \Theta^2 \Psi & \cdots & \Theta^{k-2} \Psi
        \end{bmatrix}
    \begin{bmatrix}
      \widehat{Q}_{0,0}-Q_{0,0}     & \cdots & \widehat{Q}_{0,k-2}-Q_{0,k-2}     \\
      \vdots                        &        & \vdots                            \\
      \widehat{Q}_{k-2,0}-Q_{k-2,0} & \cdots & \widehat{Q}_{k-2,k-2}-Q_{k-2,k-2}
    \end{bmatrix}
    \begin{bmatrix}
      \Psi & \Theta \Psi & \Theta^2 \Psi & \cdots & \Theta^{k-2} \Psi
    \end{bmatrix}^\top\label{eqn:decomp-2-1}           \\
     & \qquad+\begin{bmatrix}
                \Psi & \Theta \Psi & \Theta^2 \Psi & \cdots & \Theta^{k-2} \Psi
              \end{bmatrix}
    \begin{bmatrix}
      Q_{0,0}   & \cdots & Q_{0,k-2}   \\
      \vdots    &        & \vdots      \\
      Q_{k-2,0} & \cdots & Q_{k-2,k-2}
    \end{bmatrix}
    \begin{bmatrix}
      \Psi & \Theta \Psi & \Theta^2 \Psi & \cdots & \Theta^{k-2} \Psi
    \end{bmatrix}^\top.
    \label{eqn:decomp-2-2}
  \end{align}
\end{subequations}
\paragraph{Lower bound of the first term in \cref{eqn:ZZtdecomp}}

Recall \cref{eqn:Qij-defn} and 
observe that
\begin{align}\label{eqn:Qbounds}
  Q_{t,t'} =
  \begin{cases}
    0, & t\neq t', \\
    (T-t)I, & t=t'\geq 0, \\
    TI, & t=t'=-1,
  \end{cases}
\end{align}
due to \Cref{assum:iid}.
Thus, we can lower bound the term in \cref{eqn:decomp-2-2} in the following way:
\begin{align}\label{eqn:decomp-3}
  \begin{aligned}[t]
     & \begin{bmatrix}
         \Psi & \Theta \Psi & \Theta^2 \Psi & \cdots & \Theta^{k-2} \Psi
       \end{bmatrix}
    \begin{bmatrix}
      Q_{0,0}   & \cdots & Q_{0,k-2}   \\
      \vdots    &        & \vdots      \\
      Q_{k-2,0} & \cdots & Q_{k-2,k-2}
    \end{bmatrix}
    \begin{bmatrix}
      \Psi & \Theta \Psi & \Theta^2 \Psi & \cdots & \Theta^{k-2} \Psi
    \end{bmatrix}^\top                                                  \\
     & =
       \begin{bmatrix}
         \Psi & \Theta \Psi & \Theta^2 \Psi & \cdots & \Theta^{k-2} \Psi
       \end{bmatrix}
       \begin{bmatrix}
         TI     & 0          & \cdots & 0          \\
            0          & (T-1) I    & \cdots & 0          \\
            \vdots     & \vdots     & \ddots & \vdots     \\
         0          & 0          & \cdots & (T-k+2) I
       \end{bmatrix}
       \begin{bmatrix}
         \Psi & \Theta \Psi & \Theta^2 \Psi & \cdots & \Theta^{k-2} \Psi
       \end{bmatrix}^\top                                   \\
     &
       \stackrel{\mathclap{\cref{eqn:main-lem-assum-k}}}{\succeq}
       (T-\overline{t}_c+1)
       \begin{bmatrix}
         \Psi & \Theta \Psi & \Theta^2 \Psi & \cdots & \Theta^{\overline{t}_c-1} \Psi
       \end{bmatrix}
       \begin{bmatrix}
         \Psi & \Theta \Psi & \Theta^2 \Psi & \cdots & \Theta^{\overline{t}_c-1} \Psi
       \end{bmatrix}^\top
    \\
     & \stackrel{\mathclap{\cref{assum:ctrb-sigma}}}{\succeq}
        (T-\overline{t}_c+1)\beta^2 I.
  \end{aligned}
\end{align}

By \Cref{assum:tail}, the following holds with probability at least $1-\delta$:
\begin{align}\label{eqn:decomp-4}
  \left| \left| \begin{bmatrix}
                  \widehat{Q}_{0,0}-Q_{0,0}     & \cdots & \widehat{Q}_{0,k-2}-Q_{0,k-2}     \\
                  \vdots                        &        & \vdots                            \\
                  \widehat{Q}_{k-2,0}-Q_{k-2,0} & \cdots & \widehat{Q}_{k-2,k-2}-Q_{k-2,k-2}
                \end{bmatrix} \right| \right| \leq \epsilon;
\end{align}
Here, recall the definition of $\widehat{Q}_{i,j}$ and $Q_{i,j}$ in \cref{eqn:Qij-defn}.
We note that \Cref{assum:tail} provides an upper bound to the matrix in \cref{eqn:decomp-4}
appended by more rows and columns.
Since adding rows or columns to a matrix will only increase its norm,
\cref{eqn:decomp-4} also holds with probability at least $1-\delta$.
By compiling \cref{eqn:decomp-2,eqn:decomp-3,eqn:decomp-4}, we obtain the following with probability at least $1-\delta$:
\begin{align}
  \label{eqn:decomp-lower}
  \begin{aligned}
     & \begin{bmatrix}
         \Psi & \Theta \Psi & \Theta^2 \Psi & \cdots & \Theta^{k-2} \Psi
       \end{bmatrix}
    \begin{bmatrix}
      \Xi_0  \\
      \Xi_1  \\
      \vdots \\
      \Xi_{k-2}
    \end{bmatrix}
    \begin{bmatrix}
      \Xi_0  \\
      \Xi_1  \\
      \vdots \\
      \Xi_{k-2}
    \end{bmatrix}^\top
    \begin{bmatrix}
      \Psi & \Theta \Psi & \Theta^2 \Psi & \cdots & \Theta^{k-2} \Psi
    \end{bmatrix}^\top                                   \\
     & \succeq -\epsilon \sum_{\tau=0}^{k-2}\Theta^\tau \Psi\Psi^\top \Theta^{\tau\top} + (T-\overline{t}_c+1)\beta^2 I \\
     & \succeq \left((T-\overline{t}_c+1)\beta^2 - \frac{L^2 \sigma^2}{1-\alpha^2}\epsilon\right)I.
  \end{aligned}
\end{align}
The last inequality follows from the following observation, which derives from \Cref{assum:stab}:
\begin{align*}
  \|\Psi\Psi^\top \|+\|\Theta \Psi\Psi^\top\Theta^\top\| + \cdots + \|\Theta^{k-2} \Psi\Psi^\top\Theta^{(k-2)\top}\|
  & \leq \sigma^2\left(\|I\|+\|\Theta\|\|\Theta^\top\| + \cdots \right)   \\
   & \leq \sigma^2 L^2\left(1+\alpha^2+\alpha^4+\cdots\right)            \\
  & = \sigma^2L^2/(1-\alpha^2).
\end{align*}

\paragraph{Upper bound of the first term in \cref{eqn:ZZtdecomp}}

For the upper bound, we use a similar logic
as the proof of the inequality above.
Specifically, the following is obtained from \cref{eqn:decomp-2}:
\begin{align*}
   & \begin{bmatrix}
       \Psi & \Theta \Psi & \Theta^2 \Psi & \cdots & \Theta^{k-2} \Psi
     \end{bmatrix}
  \begin{bmatrix}
    \Xi_0  \\
    \Xi_1  \\
    \vdots \\
    \Xi_{k-2}
  \end{bmatrix}
  \begin{bmatrix}
    \Xi_0  \\
    \Xi_1  \\
    \vdots \\
    \Xi_{k-2}
  \end{bmatrix}^\top
  \begin{bmatrix}
    \Psi & \Theta \Psi & \Theta^2 \Psi & \cdots & \Theta^{k-2} \Psi
  \end{bmatrix}^\top                              \\
   &
     \stackrel{\mathclap{\cref{eqn:Qbounds}, \cref{assum:tail}}}{\;\;\;\preceq\;\;\;}
     (T +\epsilon)
     \begin{bmatrix}
       \Psi & \Theta \Psi & \Theta^2 \Psi & \cdots & \Theta^{k-2} \Psi
     \end{bmatrix}
     \begin{bmatrix}
       \Psi & \Theta \Psi & \Theta^2 \Psi & \cdots & \Theta^{k-2} \Psi \\
     \end{bmatrix}^\top                              \\
   &
     \stackrel{\mathclap{\cref{assum:bounded-cov}}}{\;\;\;\preceq\;\;\;}
     (T+\epsilon)\sigma^2
  \begin{bmatrix}
    I& \Theta  & \Theta^2  & \cdots & \Theta^{k-2} 
  \end{bmatrix}
  \begin{bmatrix}
    I& \Theta  & \Theta^2  & \cdots & \Theta^{k-2}  \\
  \end{bmatrix}^\top                              \\
   & \;\;\;\preceq\;\;\; (T+\epsilon)\sigma^2
     (\| I\| + \|\Theta\Theta^\top\| + \cdots + \|\Theta^{k-2}\Theta^{(k-2)\top}\|)I \\
   &
     \stackrel{\mathclap{\cref{assum:stab}}}{\;\;\;\preceq\;\;\;}
     \frac{\sigma^2 L^2}{1-\alpha^2}(T+\epsilon)I .
\end{align*}

\paragraph{Bound of the second term in \cref{eqn:ZZtdecomp}}
By the definition of $\widehat{Q}_{t,t}$ in \cref{eqn:Qij-defn}, 
one can see that for all $t=0,1,\cdots,T-1$,
\begin{align*}
  \|\Xi_{t}\| = \|\widehat{Q}_{t,t}\|^{1/2}
\end{align*}
and
\begin{align*}
  \widehat{Q}_{0,0} \succeq \widehat{Q}_{t,t}.
\end{align*}
Therefore,
\begin{align*}
  \|\Xi_{t}\|
   & = \|\widehat{Q}_{t,t}\|^{1/2}                            \\
   & \leq \|\widehat{Q}_{0,0}\|^{1/2}                         \\
   & \leq (\|\widehat{Q}_{0,0}-Q_{0,0}\| + \|Q_{0,0}\|)^{1/2}
\end{align*}
By \Cref{assum:tail},
$\|\widehat{Q}_{0,0}-Q_{0,0}\| \leq \epsilon$, and
by \cref{eqn:Qbounds},
$$\|Q_{0,0}\| \leq T.$$
Thus, we have for all $t=0,1,\cdots,T-1$,
\begin{align}
  \|\Xi_{t}\| \leq (\epsilon + T)^{1/2}.
  \label{eqn:ximag}
\end{align}
This implies
\begin{align}\label{eqn:upper-1}
  \begin{aligned}
    \left\|
    \begin{bmatrix}
      \Theta^{k-1} \Psi& \Theta^k \Psi& \cdots & \Theta^{T-1}\Psi
    \end{bmatrix}
    \begin{bmatrix}
      \Xi_{k-1}   \\
      \Xi_k \\
      \vdots    \\
      \Xi_{T-1}
    \end{bmatrix}
    \right\| 
    & \leq \|\Theta^{k-1} \Psi \Xi_{k-1} \|+ \cdots + \|\Theta^{T-1} \Psi \Xi_{T-1}\| \\
               &\stackrel{\mathclap{\cref{eqn:ximag}}}{\leq}(\epsilon +T)^{1/2}\|\Psi\|(\|\Theta^{k-1}\| + \cdots + \|\Theta^{T-1}\|) \\
               &\stackrel{\mathclap{\cref{assum:stab}}}{\leq} (\epsilon +T)^{1/2}\frac{\sigma L\alpha^{k-1}}{1-\alpha}.
  \end{aligned}
\end{align}
Furthermore,
\begin{align}\label{eqn:upper-2}
  \begin{aligned}
    \left\|
    \begin{bmatrix}
      \Psi & \Theta \Psi & \Theta^2 \Psi & \cdots & \Theta^{k-2} \Psi
    \end{bmatrix}
    \begin{bmatrix}
      \Xi_0  \\
      \Xi_1  \\
      \vdots \\
      \Xi_{k-2}
    \end{bmatrix} 
    \right\| 
    &\leq \|\Xi_0\|\|\Psi\| + \|\Xi_1\|\|\Theta\|\|\Psi\| + \cdots + \|\Xi_{k-2}\|\|\Theta^{k-2}\|\|\Psi\| \\
    &\stackrel{\mathclap{\cref{eqn:ximag},\cref{assum:bounded-cov}}}{\leq} (\epsilon + T)^{1/2} \sigma (1+\|\Theta\|+\cdots +\|\Theta^{k-2}\|)\\
    &
    \stackrel{\mathclap{\cref{assum:stab}}}{\leq}
    (\epsilon + T)^{1/2} \frac{\sigma L}{1-\alpha}.
  \end{aligned}
\end{align}
Finally, we obtain the following upper bound from \cref{eqn:upper-1,eqn:upper-2}:
\begin{align}\label{eqn:upper-3}
  \begin{aligned}
     & \left\|
    \begin{bmatrix}
      \Theta^{k-1} \Psi& \Theta^k \Psi& \cdots & \Theta^{T-1}\Psi
    \end{bmatrix}
    \begin{bmatrix}
      \Xi_{k-1}   \\
      \Xi_k \\
      \vdots    \\
      \Xi_{T-1}
    \end{bmatrix}
    \begin{bmatrix}
      \Xi_0  \\
      \Xi_1  \\
      \vdots \\
      \Xi_{k-2}
    \end{bmatrix}^\top
    \begin{bmatrix}
      \Psi & \Theta \Psi & \Theta^2 \Psi & \cdots & \Theta^{k-2} \Psi
    \end{bmatrix}^\top
    \right\|                                                                                       \\
     & \leq (\epsilon + T)^{1/2}\frac{\sigma L\alpha^{k-1}}{1-\alpha} (\epsilon + T)^{1/2} \frac{\sigma  L}{1-\alpha} \\
     & = (\epsilon + T)\frac{\sigma^2 L^2\alpha^{k-1}}{(1-\alpha)^2}.
  \end{aligned}
\end{align}

\paragraph{Putting everything together}

From \cref{eqn:upper-3}, we have
\begin{align}\label{eqn:cross-term}
  \begin{aligned}
    \begin{bmatrix}
      \Theta^{k-1} \Psi& \Theta^k \Psi& \cdots & \Theta^{T-1}\Psi
    \end{bmatrix}
    \begin{bmatrix}
      \Xi_{k-1}   \\
      \Xi_k \\
      \vdots    \\
      \Xi_{T-1}
    \end{bmatrix}
    \begin{bmatrix}
      \Xi_0  \\
      \Xi_1  \\
      \vdots \\
      \Xi_{k-2}
    \end{bmatrix}^\top
    \begin{bmatrix}
      \Psi & \Theta \Psi & \Theta^2 \Psi & \cdots & \Theta^{k-2} \Psi
    \end{bmatrix}^\top   \\
    +
    \begin{bmatrix}
      \Psi & \Theta \Psi & \Theta^2 \Psi & \cdots & \Theta^{k-2} \Psi
    \end{bmatrix}
    \begin{bmatrix}
      \Xi_0  \\
      \Xi_1  \\
      \vdots \\
      \Xi_{k-2}
    \end{bmatrix}
    \begin{bmatrix}
      \Xi_{k-1}   \\
      \Xi_k \\
      \vdots    \\
      \Xi_{T-1}
    \end{bmatrix}^\top
    \begin{bmatrix}
      \Theta^{k-1} \Psi& \Theta^k \Psi& \cdots & \Theta^{T-1}\Psi
    \end{bmatrix}^\top \\
    \succeq - 2(\epsilon + T)\frac{\sigma^2 L^2\alpha^{k-1}}{(1-\alpha)^2}I.
  \end{aligned}
\end{align}
Additionally,
we additionally observe that
\begin{align}\label{eqn:nonneg}
  \begin{aligned}
    \begin{bmatrix}
      \Theta^{k-1} \Psi& \Theta^k \Psi& \cdots & \Theta^{T-1}\Psi
    \end{bmatrix}
    \begin{bmatrix}
      \Xi_{k-1}   \\
      \Xi_k \\
      \vdots    \\
      \Xi_{T-1}
    \end{bmatrix}
    \begin{bmatrix}
      \Xi_{k-1}   \\
      \Xi_k \\
      \vdots    \\
      \Xi_{T-1}
    \end{bmatrix}^\top
    \begin{bmatrix}
      \Theta^{k-1} \Psi& \Theta^k \Psi& \cdots & \Theta^{T-1}\Psi
    \end{bmatrix}^\top
    \succeq 0.
  \end{aligned}
\end{align}
By compiling \cref{eqn:ZZtdecomp,eqn:decomp-lower,eqn:cross-term,eqn:nonneg}, we obtain
\begin{align}\label{eqn:key-bound-1}
  Z Z^\top \succeq
  \left((T-\overline{t}_c+1)\beta^2 - \epsilon \frac{\sigma^2 L^2}{1-\alpha^2} - 2(\epsilon + T)\frac{\sigma^2 L^2\alpha^{k-1}}{(1-\alpha)^2}\right)I.
\end{align}
By \cref{eqn:main-lem-assum}, we have that
\begin{align*}
  T-\overline{t}_c+1 & \geq T/2\\
  \frac{2\sigma^2 L^2}{(1-\alpha)^2}\alpha^{k-1} & \leq \frac{\beta^2}{8}\\
  \epsilon \left(
  \frac{\sigma^2 L^2}{1-\alpha^2} + \frac{\beta^2}{8}\right)&\leq \frac{T\beta^2}{8},
\end{align*}
yielding
\begin{align}\label{eqn:key-bound-1-1}
  Z Z^\top \succeq \frac{T\beta^2}{4} I.
\end{align}

\subsection{Upper bound of $\overline\Xi Z^\top$} \label{subsec:upper-bound-XiZT}
Recall the definition of $\overline\Xi$ in \cref{eqn:matform} 
and that of $\Xi_t$ in \cref{eqn:Qij-defn}, we have
$\overline\Xi = \Psi \Xi_{-1}$, so
\begin{align*}
  \overline\Xi Z^\top
   & = \overline\Xi
  \begin{bmatrix}
    \Xi_0  \\
    \Xi_1  \\
    \vdots \\
    \Xi_{T-1}
  \end{bmatrix}^\top
  \begin{bmatrix}
    \Psi & \Theta \Psi& \Theta^2 \Psi& \cdots & \Theta^{T-1} \Psi
  \end{bmatrix}^\top                                                             \\
   & =
  \begin{bmatrix}
    \Xi_{0} \Xi_{-1}^\top \Psi^\top\\
    \Xi_{1} \Xi_{-1}^\top \Psi^\top\\
    \vdots                  \\
    \Xi_{T-1} \Xi_{-1}^\top \Psi^\top
  \end{bmatrix}^\top
  \begin{bmatrix}
    \Psi & \Theta \Psi& \Theta^2 \Psi& \cdots & \Theta^{T-1}\Psi
  \end{bmatrix}^\top                                                             \\
   & =\sum_{t=0}^{T-1}  \Psi \widehat{Q}_{-1,t} \Psi^\top (\Theta^{t})^\top                                                       \\
\end{align*}
We can decompose $\overline\Xi Z^\top$ into the following,
\begin{align}
  \overline\Xi Z^\top = 
  \sum_{t=0}^{k-2} \Psi \widehat{Q}_{-1,t} \Psi^\top (\Theta^{t})^\top 
  + 
  \sum_{t=k-1}^{T-1} \Psi \widehat{Q}_{-1,t} \Psi^\top (\Theta^{t})^\top,
  \label{eqn:XiZ-decompose}
\end{align}
and we bound the two terms separately.
\paragraph{Upper bound of the first term in \cref{eqn:XiZ-decompose}}
\begin{align}\label{eqn:upper-decomp}
  \begin{aligned}
  \left\|\sum_{t=0}^{k-2} \Psi \widehat{Q}_{-1,t} \Psi^\top (\Theta^{t})^\top  \right\|
  &\leq \left\|\sum_{t=0}^{k-2}  \Psi (\widehat{Q}_{-1,t} - Q_{-1,t}) \Psi^\top (\Theta^{t})^\top + \sum_{t=0}^{k-2}  \Psi Q_{-1,t} \Psi^\top (\Theta^{t})^\top\right\|  \\
  &\leq \sum_{t=0}^{k-2}  \|(\widehat{Q}_{-1,t} - Q_{-1,t})\| \|\Theta^{t}\|\| \Psi\|^2 + \sum_{t=0}^{k-2}  \|Q_{-1,t}\| \|\Theta^{t}\|\| \Psi\|^2,
  \end{aligned}
\end{align}
by triangle and sub-multiplicativity inequality.
We then invoke \Cref{assum:main,assum:tail} to obtain
\begin{align}
  \begin{aligned}
    \sum_{t=0}^{k-2}  \|(\widehat{Q}_{-1,t} - Q_{-1,t})\| \|\Theta^{t}\| \|\Psi\|^2& 
    \leq \epsilon \sum_{t=0}^{k-2}  L\alpha^t \sigma^2, \\
    & \leq  \epsilon \frac{\sigma^2 L}{1-\alpha}.
  \end{aligned}
  \label{eq:decomp2-3-1}
\end{align}
Additionally,
since $\{\xi_t\}$ is an independent sequence with zero mean (due to \Cref{assum:main}), we have
\begin{align}
  \begin{aligned}
    \sum_{t=0}^{k-2}  \|Q_{-1,t}\| \|\Theta^{t}\| = 0.
  \end{aligned}
  \label{eq:decomp2-3-2}
\end{align}
Thus, combining \cref{eqn:upper-decomp,eq:decomp2-3-1,eq:decomp2-3-2},
we obtain
\begin{align}
  \left\|\sum_{t=0}^{k-2} \Psi \widehat{Q}_{-1,t} \Psi^\top (\Theta^{t})^\top  \right\| 
  \leq \epsilon \frac{\sigma^2 L}{1-\alpha}.
  \label{eq:decomp2-3-3}
\end{align}
\paragraph{Upper bound of the second term in \cref{eqn:XiZ-decompose}}
We can not invoke \Cref{assum:tail} for the second term of \cref{eqn:XiZ-decompose} 
due to $t \geq k-1$, 
so we bound $\Xi_{-1}$ and $\Xi_{t}$ separately. 
Similar to \cref{eqn:ximag}, we can upper bound
\begin{align} \label{eqn:ximag-1}
  \|\Xi_{-1}\| \leq (\epsilon + T)^{1/2}
\end{align}
by \Cref{assum:iid} and \Cref{assum:tail}.
By triangle and sub-multiplicativity inequality,
\begin{align}
  \begin{aligned}
  \left\|\sum_{t=k-1}^{T-1} \Psi \widehat{Q}_{-1,t} \Psi^\top (\Theta^{t})^\top\right\| 
  &\stackrel{\mathclap{\cref{assum:bounded-cov}}}{\;\;\;\leq\;\;} 
  \sigma^2 \sum_{t=k-1}^{T-1} \|\Xi_{-1}\| \|\Xi_{t}\| \|\Theta^{t}\| \\
  &\stackrel{\mathclap{\cref{eqn:ximag}, \cref{eqn:ximag-1}}}{\;\;\;\leq\;\;} \sigma^2 (\epsilon + T) \sum_{t=k-1}^{T-1} \|\Theta^{t}\| \\
  &\;\;\;\leq\;\; \sigma^2 (\epsilon + T)\frac{L\alpha^{k-1}}{1-\alpha}
  \end{aligned}
  \label{eq:decomp2-3-4}
\end{align}
\paragraph{Putting everything together}
Combining \cref{eq:decomp2-3-3,eq:decomp2-3-4}, we obtain
\begin{align}
  \begin{aligned}
  \left\|\overline\Xi Z^\top\right\|&\leq \epsilon \frac{\sigma^2 L}{1-\alpha} + (\epsilon + T)\frac{\sigma^2L\alpha^{k-1}}{1-\alpha} \\
  &\stackrel{\mathclap{\cref{eqn:main-lem-assum}}}{\leq} \epsilon (1+\alpha^{k-1})\frac{\sigma^2 L}{1-\alpha} + \frac{\sigma^2L}{1-\alpha}T^{1/2} \\
  & \leq \frac{\sigma^2L}{1-\alpha} (2\epsilon + T^{1/2})
  \end{aligned}
  \label{eqn:key-bound-2}
\end{align}

\subsection{Final bound} \label{subsec:lemma-final-bound}
Finally, by combining \cref{eqn:ls-decomp,eqn:key-bound-1-1,eqn:key-bound-2}, we have that with probability at least $1-\delta$, the estimator is unique (due to \cref{eqn:key-bound-1-1}) and obtain the bound in \cref{eqn:lem-main-res}.

\section{Sample Complexity of sub-Gaussian Distribution} 
\label{sec:thm-subgauss}
This section considers sub-Gaussian noise distributions. 
We provide this result to demonstrate the general applicability of \Cref{lem:main} 
and to provide a baseline for comparison with other distribution classes. 
First, we provide the definition of sub-Gaussian distributions.

\begin{defn}
  An $r$-dimensional, zero-mean \gls*{rv} $X$ is called $\rho$-sub-Gaussian with $\rho\in\mathbb{R}_{>0}$ if 
  $$\mathbb{E}\left[e^{\lambda\langle v, X\rangle}\right] \leq e^{\rho^2\lambda^2/2},\quad \forall
  \lambda \in \mathbb{R},\quad \forall v \in \mathbb{S}^{r-1}.$$ 
\end{defn}

\newcommand{\epsprefactor}{
  
}
\newcommand{\subgaussepsilon}{
  \rho^2 k \left(
    4c_1 k +
    2\max\left(c_3^{-1}\log\left(\frac{2c_2 k}{\delta}\right), 1\right)
  \right)r^{1/2}T^{1/2}
}
\newcommand{\subgausslembounds}{
  4k,
  r,
  \frac{2k}{c_3}\log\left(\frac{2c_2 k}{\delta}\right)
}

\newcommand{\subgaussburninT}{
  \max\Bigg\{
    &
      \subgausslembounds,
      2\overline{t}_c,\\
    &
      \left(
      \frac{8k \rho^2}{\beta^2} \left(
      4c_1 k +
      2\max\left(c_3^{-1}\log\left(\frac{2c_2 k}{\delta}\right), 1\right)
      \right)\left(\frac{\sigma^2 L^2}{1-\alpha^2} + \frac{\beta^2}{8}\right)r^{1/2}
      \right)^2
  \Bigg\}
}

\newcommand{\subgaussfinalbound}{\frac{ 8 \sigma^2 L }{ (1-\alpha) \beta^2 } \left(
  k\rho^2 \left(
    4c_1 k +
    2\max\left(c_3^{-1}\log\left(\frac{2c_2 k}{\delta}\right), 1\right)
  \right)
  +
  1
  \right)r^{1/2}T^{-1/2}
  }

We are now ready to state the sample complexity result for sub-Gaussian noise. 

\begin{thm}[Sample Complexity for sub-Gaussian Noise]\label{thm:subgauss} 
  Suppose that \Cref{assum:main} holds, 
  and $\xi_t$ is $\rho$-sub-Gaussian with $\rho\in\mathbb{R}_{>0}$, for all $t=0,1,\cdots,T$.
  There exist universal polynomials $\polyGburn(\cdot), \polyGbound(\cdot)$ such that
  for given $\delta\in(0,1)$, if
  \begin{align} 
    \begin{aligned}[t]
      T \geq \polyGburn\left(\mathcal{S} \cup \{r, \log(1/\delta), \rho \}\right),
    \end{aligned}
    \label{eqn:subgauss-assum-poly}
  \end{align}
  the following holds with probability at least $1-\delta$:
  \begin{align}
    \left\|\widehat{\Theta} - \Theta\right\| \leq \polyGbound\left( \mathcal{S} \cup \{\log(1/\delta), \rho, \log(T)\}\right)r^{1/2}T^{-1/2}
  \label{eqn:subgauss-result-poly}
  \end{align}
  where
  $r,T,\Theta$ are defined in \cref{eqn:sys},
  $\widehat{\Theta}$ is defined in \cref{eqn:est},
  $L, \alpha, \beta, \sigma, \overline{t}_c$ are defined in \Cref{assum:main},
  $k$ is defined in \cref{eqn:main-lem-assum-k}, and $\mathcal{S}$ is defined in \cref{eqn:cal-S}.
\end{thm}

One can observe that the estimation error is $\widetilde{\bigO}(r^{1/2}T^{-1/2})$ with probability at least $1-\delta$ 
for sufficiently large $T$, 
where the constant factor depends polynomially on $\mathcal{S}\cup\{\log(1/\delta), \rho, \log(T)\}$.
Our result is similar to the result stated by \cite{ziemannTutorialNonAsymptoticTheory2024}
for stable systems, except for an additional $\log(T)$ depedence.
This can be avoided with additional steps applying the exponentially stable property in \Cref{lem:main}
for sub-Gaussian distributions, but unavoidable 
for sub-exponential distributions or heavy-tailed distributions. 
Since we focus primarily on heavy-tailed noise, we do not pursue this idea further.

Our proof involves finding the explicit form 
of $\epsilon$ in \Cref{lem:sub-gausstail}, which is obtained by applying the standard concentration bound 
for sub-Gaussian ensembles \cite[theorem 6.5]{wainwrightHighDimensionalStatisticsNonAsymptotic2019} 
to each term in \cref{eqn:emp-cov-decomp}.
We then apply $\epsilon$ explicitly to \Cref{lem:main} 
for deriving the final sample complexity bound.
Specifically, $T$ is defined to be sufficiently large when it satisfies \cref{eqn:subgauss-assum-poly},
or more specifically \cref{eqn:subgauss-burninT}, shown in the proof of \Cref{thm:subgauss} below.
The first two terms ensures that there are enough samples 
to learn the system of noise dimension $r$ with adequate $k$ blocking.
The next term guarantees the applicability of the tail bound using concentration inequality,
and the last two term ensures the conditions in \cref{eqn:main-lem-assum-eps} 
is satisfied to apply \Cref{lem:main}.

We now discuss the formal proof. First, we introduce the following lemma 
on the concentration of the empirical covariance 
of the augmented noise vector in \Cref{assum:tail} 
when the individual noise vectors are sub-Gaussian.
\begin{lem}[sub-Gaussian Tail Bound for sub-Gaussian \glspl*{rv}]\label{lem:sub-gausstail}
  Under \Cref{assum:iid}
  and $\delta\in(0,1)$, if $\xi_t$ is $\rho$-sub-Gaussian for all $t=0,1,\cdots,T$,
  and
  \begin{align} \label{eqn:subgauss-lem-assum} 
    T \geq \max\left\{
    \subgausslembounds
    \right\}
  \end{align}
  the following holds with probability at least $1-\delta$:
  \begin{align} \label{eqn:subgauss-prop-result}
    \begin{aligned}[t]
      &\left\|\tail\right\|\\
      &\qquad\leq \subgaussepsilon,
    \end{aligned}
  \end{align}
  where $k$ is defined in \cref{eqn:main-lem-assum-k}, 
  $c_1,c_2,c_3\in\mathbb{R}_{>0}$ are universal constants 
  introduced in \cite[theorem 6.5]{wainwrightHighDimensionalStatisticsNonAsymptotic2019}.
\end{lem}
\begin{proof}  
We first observe a few properties of $T^{(k,i)}$ defined in \cref{eqn:emp-cov-decomp}:
\begin{align}\label{eqn:Tki-bound} 
  T/k-2 \leq T^{(k,i)}\leq T/k. 
\end{align}
Since $T\geq 4k$ from \cref{eqn:subgauss-lem-assum}, 
we have the following simpler lower bound:
\begin{align}\label{eqn:Tki-prac-bound}
  T^{(k,i)} \geq T/2k.
\end{align}
We individually bound $\Delta_1$ and $\Delta_2$ in \cref{eqn:emp-cov-decomp}. 
One can see that the time series $\{\bxi^{(k,i)}_\tau\}_{\tau=1,2,\cdots}$ is independent for each $i=0,1,\cdots,k-1$, and their covariance is identity due to the bounded covariance condition in \Cref{assum:iid}. That is,
\begin{align}\label{eqn:unitcov}
  \mathbb{E}[\bxi^{(k,i)}_\tau (\bxi^{(k,i)}_\tau)^\top]
  \stackrel{{\cref{assum:iid}}}{=}
  I
\end{align}
.

\paragraph{Upper Bound on $\Delta_{1,k}$}
Consider the following constant:
\begin{align}\label{eqn:subgauss-phi}
  \phi= \left(\frac{1}{c_3 T^{(k,i)}}\log\left(\frac{2c_2 k}{\delta}\right)\right)^{1/2},
\end{align}
which satisfies the following property:
\begin{align}\label{eqn:subgauss-phi-observe}
  \begin{aligned}[t]
    \phi^2 &\stackrel{\mathclap{\cref{eqn:Tki-prac-bound},\cref{eqn:subgauss-lem-assum}}}{\;\;\;\leq\;\;\;} \frac{2k}{c_3 T}\log\left(\frac{2c_2 k}{\delta}\right)\\
           &\stackrel{\mathclap{\cref{eqn:subgauss-lem-assum}}}{\;\;\;\leq\;\;\;} 1.
  \end{aligned}
\end{align}
By the concentration bound for $\rho$-sub-Gaussian ensembles
\cite[theorem 6.5]{wainwrightHighDimensionalStatisticsNonAsymptotic2019} and the bounded covariance condition in \cref{eqn:unitcov},
we obtain the following:
\newcommand{\termoneupperbound}{
  \rho^2 \left(2c_1k + c_3^{-1/2}\log^{1/2}\left(\frac{2c_2 k}{\delta}\right) \right)r^{1/2}T^{1/2}
}
\begin{align} 
  \begin{aligned}[t]
    &\left\|\termonemono\right\| \\
    &\;\;\leq\;\; 
      \rho^2 T^{(k,i)}\left(
      c_1\left(\sqrt{\frac{kr}{T^{(k,i)}}} + \frac{kr}{T^{(k,i)}}\right) + \phi
    \right) \\
    &\stackrel{\mathclap{\cref{eqn:subgauss-phi}}}{\;\;\leq\;\;} \rho^2 \left(c_1kr + \left(c_1k^{1/2}r^{1/2} + c_3^{-1/2}\log^{1/2}\left(\frac{2c_2 k}{\delta}\right)\right)\left(T^{(k,i)}\right)^{1/2} \right)\\
    &
      \stackrel{\mathclap{\cref{eqn:Tki-bound}}}{\;\;\leq\;\;}
      \rho^2 \left(c_1kr + \left(c_1k^{1/2}r^{1/2} + c_3^{-1/2}\log^{1/2}\left(\frac{2c_2 k}{\delta}\right)\right)\left(\frac{T}{k}\right)^{1/2} \right)
    \\
    &
      \stackrel{\mathclap{\cref{eqn:subgauss-lem-assum},r\geq 1}}{\;\;\leq\;\;}
      \rho^2 \left(c_1kr^{1/2}T^{1/2} + \left(c_1k^{1/2}r^{1/2} + c_3^{-1/2}r^{1/2}\log^{1/2}\left(\frac{2c_2 k}{\delta}\right)\right)\left(\frac{T}{k}\right)^{1/2} \right)
    \\
    &
      \stackrel{k\geq 1}{\;\;\leq\;\;}
      \termoneupperbound
\end{aligned} 
\label{eqn:subgauss-bound1}
\end{align}
with probability at least
\begin{align}\label{eqn:term1-tail}
  \begin{aligned}[t]
    1-c_2e^{-c_3 T^{(k,i)}\min\{\phi, \phi^2\}}
    & \stackrel{\mathclap{\cref{eqn:subgauss-phi-observe}}}{=} 1-c_2e^{-c_3 T^{(k,i)}\phi^2}\\
    & \stackrel{\mathclap{\cref{eqn:subgauss-phi}}}{=} 1-\frac{\delta}{2k},
  \end{aligned}
\end{align}
 where $c_1,c_2,c_3\in\mathbb{R}_{>0}$ are universal constants 
 introduced in \cite[theorem 6.5]{wainwrightHighDimensionalStatisticsNonAsymptotic2019}.

\paragraph{Bounds on $\Delta_{2,k}$}
Fix the constant
\begin{align}\label{eqn:subgauss-phi-prime}
  \phi' = \max\left(c_3^{-1}\log\left(\frac{2c_2 k}{\delta}\right), 1\right),
\end{align}
which satisfies
\begin{align}\label{eqn:subgauss-phi-prime-observe}
  \min(\phi', \phi'^2) \geq c_3^{-1}\log\left(\frac{2c_2 k}{\delta}\right),
\end{align}
where we have $\phi' \geq 1$ from \cref{eqn:subgauss-phi-prime}, so $ \min(\phi', \phi'^2) = \phi'$,
and the inequality follows.
Similar to $\Delta_{1,k}$, we bound the $\Delta_{2,k}$ with 
the concentration bound for sub-Gaussian ensembles
\cite[theorem 6.5]{wainwrightHighDimensionalStatisticsNonAsymptotic2019}  and the unit covariance obtained from \cref{eqn:unitcov}:
\newcommand{\termtwoupperbound}{
  \rho^2\left(
    2c_1k +
    \max\left(c_3^{-1}\log\left(\frac{2c_2 k}{\delta}\right), 1\right)
  \right) r^{1/2}T^{1/2}}
\begin{align} 
  \begin{aligned}[t]
    &\left\|\termtwomono\right\|\\
      &\;\;\leq\;\; \rho^2\left(c_1\left\{\sqrt{kr} + kr\right\} + \phi' \right)\\
    &\stackrel{\mathclap{\cref{eqn:subgauss-phi-prime}}}{\;\;\leq\;\;}
      \rho^2\left(
    c_1kr +
    c_1k^{1/2}r^{1/2}+
    \max\left(c_3^{-1}\log\left(\frac{2c_2 k}{\delta}\right), 1\right)
  \right)\\
    &\stackrel{\mathclap{\cref{eqn:subgauss-lem-assum},r\geq 1}}{\;\;\leq\;\;}
        \rho^2\left(
    c_1k +
    c_1k^{1/2}+
      \max\left(c_3^{-1}\log\left(\frac{2c_2 k}{\delta}\right), 1\right)
      \right) r^{1/2}T^{1/2}\\
    &\stackrel{k\geq 1}{\;\;\leq\;\;}
      \termtwoupperbound      
\end{aligned}
\label{eqn:subgauss-bound2}
\end{align}
with probability at least
\begin{align}\label{eqn:term2-tail}
  \begin{aligned}[t]
    1-c_2e^{-c_3 \min\{\phi', \phi'^2\}}
    & \stackrel{\mathclap{\cref{eqn:subgauss-phi-prime},\cref{eqn:subgauss-phi-prime-observe}}}{\;\;\;\geq\;\;\;} 1-c_2e^{-\log\left(\frac{2c_2 k}{\delta}\right)} \\
    & \;\;\;=\;\;\; 1-\frac{\delta}{2k},
  \end{aligned}
\end{align}

\paragraph{Combining terms together}
Combining $\Delta_{1,k}$ and $\Delta_{2,k}$ and using the triangle inequality,
we obtain
\begin{align}\label{eqn:subGauss-bound}
  \begin{aligned}[t]
    &\left\|\tail\right\|\\
    &\stackrel{\mathclap{\cref{eqn:subgauss-bound1},\cref{eqn:subgauss-bound2}}}{\;\;\;\leq\;\;\;}
      \sum_{i=0}^{k-1}\termoneupperbound
    \\
    &\qquad +
      \sum_{i=0}^{k-1}\termtwoupperbound \\
    &\;\;\;\leq\;\;\;
      \subgaussepsilon,
  \end{aligned}
\end{align}
with probability at least 
\begin{align*}
  1-\sum_{i=0}^{k-1} \frac{\delta}{2k} - \sum_{i=0}^{k-1} \frac{\delta}{2k} = 1-\delta
\end{align*}
due to \cref{eqn:term1-tail,eqn:term2-tail}.
\end{proof}



Based on \Cref{lem:sub-gausstail,lem:main}, we are now ready to prove \Cref{thm:subgauss}.
{
  \def\proofname{Proof of \Cref{thm:subgauss}}
  \begin{proof}
    We choose $k$ as defined in \cref{eqn:main-lem-assum-k},
    and suppose $T$ is satisfies the following,
    \begin{align} \label{eqn:subgauss-burninT}
      \begin{aligned}
      T \geq \subgaussburninT.
      \end{aligned}
    \end{align}
    We seek to apply \Cref{lem:sub-gausstail} since it provides an upper bound on $\epsilon$.
    Assumptions of \Cref{lem:sub-gausstail} are satisfied due to \cref{eqn:subgauss-burninT},
    so we have
    \begin{align}
      &\begin{aligned}[t]
        &\epsilon \leq \subgaussepsilon
      \end{aligned}
    \end{align}
    with probability at least $1-\delta$. 
    We now provide the final bound using \Cref{lem:main}.
    We observe that 
    \begin{align*}
       \subgaussepsilon\leq \epsbound,
    \end{align*}
    by \cref{eqn:subgauss-burninT}. 
    Therefore,
    \Cref{assum:tail} is satisfied with $\epsilon$ satisfying \cref{eqn:main-lem-assum-eps}.
    Finally,
    \cref{eqn:main-lem-assum-T} is satisfied due to \cref{eqn:subgauss-burninT}, confirming
    that all conditions in \cref{eqn:main-lem-assum} hold.
    Therefore, we may apply \Cref{lem:main}, which indicates that
    \begin{align} \label{eqn:subgauss-final-result}
      \left\|\widehat{\Theta} - \Theta\right\|
      &\leq  \subgaussfinalbound.
    \end{align}
    By inspecting \cref{eqn:subgauss-burninT,eqn:subgauss-final-result}, 
    one can see that there exist universal polynomials $\polyGburn(\cdot),\polyGbound(\cdot)$ 
    such that \cref{eqn:subgauss-burninT,eqn:subgauss-final-result} hold under \cref{eqn:subgauss-assum-poly,eqn:subgauss-result-poly}.
    Specifically, the terms in $\polyGburn(\cdot)$ comes from breaking down the terms in \cref{eqn:subgauss-burninT}
    and replacing $k$ with its polynomial equivalence from the definition \cref{eqn:main-lem-assum-k}. 
    We note that $k$ depends polynomially on $\log(T)$. 
    Since $\log(T)$ grows more slowly than any positive power of $T$, 
    this dependence can be absorbed into a polynomial dependence on $T$ for sufficiently large $T$.
    $\polyGbound(\cdot)$ in \cref{eqn:subgauss-result-poly} can also be 
    derived from \cref{eqn:subgauss-final-result} using the same approach,
    thus completing the proof.
  \end{proof}
}

\section{Sub-exponential distributions}\label{sec:thm-subexp}
We establish the sample complexity bound for sub-exponential noise distributions
in this section. We refer to \cite[Definition 2.7]{wainwrightHighDimensionalStatisticsNonAsymptotic2019}  
for the definition of sub-exponential distributions stated below:
\begin{defn} \label{def:sub-exponential}
  An $r$-dimensional, zero-mean \gls*{rv} $X$ is called $(\nu,\zeta)$-sub-exponential with $\nu,\zeta \in\mathbb{R}_{>0}$ if
  \begin{align*}
    \mathbb{E}[e^{\lambda\langle X, u\rangle}] \leq e^{\nu^2 \lambda^2/2},\quad\forall u\in\mathbb{S}^{r-1},\quad |\lambda|\leq \frac{1}{\zeta}.
  \end{align*}
\end{defn}

\newcommand{\subexplembounds}{
  (6 + \sqrt{2})^8,
  4k,
  r,
  \frac{\delta e^{\nu^2/2\zeta^2}}{4rk},
  \frac{\delta e^{1/4\zeta}}{4rk},
  1024 \zeta^4 k^4 r^2\log^6\left(\frac{8rk^2}{\delta}\right)
}
\newcommand{\subexpthmboundone}{
  \frac{2^{24}\zeta^4 k^4}{\beta^8} \log^{6}\left(\frac{8rk^2}{\delta}\right) \left(\frac{\sigma^2 L^2}{1-\alpha^2} + \frac{\beta^2}{8}\right)^{4}r^{2}
}
\newcommand{\subexpthmboundtwo}{
  \frac{2^{16}\zeta^4 k^4}{\beta^4}\log^6\left(\frac{8rk^2}{\delta}\right)\left(\frac{\sigma^2 L^2}{1-\alpha^2} + \frac{\beta^2}{8}\right)^2 r^{2}
}
\newcommand{\subexpfinalbound}{
  \frac{ 8 \sigma^2 L}{ (1-\alpha) \beta^2 } \left(
  \left(
    1 + 4k\zeta \log\left(\frac{4rkT}{\delta}\right)\log^{1/2}\left(\frac{8rk^2}{\delta}\right)
  \right)^2 + 1 
  \right)
  r^{1/2}T^{-1/2}
}
We now introduce our first theorem on sample complexity bound for sub-exponential noise.
\begin{thm}[Sample Complexity for Sub-Exponential Noise]\label{thm:subexp}
  Suppose that \Cref{assum:main} holds,
  and $\xi_t$ is $(\nu,\zeta)$-sub-exponential with $\nu,\zeta\in \mathbb{R}_{>0}$, 
  for all $t=0,1,\cdots,T$. 
  There exist universal polynomials $\polyEburn(\cdot)$ and $\polyEbound(\cdot)$ such that for given $\delta\in(0,1)$, if 
  \begin{align}
    \label{eqn:subexp-thm-assum-poly}
    T \geq \polyEburn\left(\mathcal{S} \cup  \left\{r, \log(1/\delta), e^{\nu^2/2\zeta^2}, e^{1/4\zeta}, \zeta \right\}\right),
  \end{align}
  then the following holds with probability at least $1-\delta$:
  \begin{align}
    \left\|\widehat{\Theta} - \Theta\right\|
    \leq
    \polyEbound\left(\mathcal{S} \cup \left\{ \log(T), \log(r), \log(1/\delta), \zeta\right\}\right) r^{1/2}T^{-1/2},
    \label{eqn:subexp-thm}
  \end{align}
  where
  $r,T,\Theta$ are defined in \cref{eqn:sys},
  $\widehat{\Theta}$ is defined in \cref{eqn:est},
  $L, \alpha, \beta, \sigma, \overline{t}_c$ are defined in \Cref{assum:main},
  and $\mathcal{S}$ is defined in \cref{eqn:cal-S}.
\end{thm}

Our estimation error is bounded by $\widetilde{\bigO}((r/T)^{1/2})$ 
with probability at least $1-\delta$ for sufficiently large $T$. 
This result can be interpreted as the following. 
For sufficiently large $T$,
when $\delta \rightarrow 0$,
the estimate error becomes larger,
which comes from $\log(1/\delta)$ term in \cref{eqn:subexp-thm}.
This means that the failure probability is higher
for \gls*{ols} to produce a parameter estimate that has a better accuracy
with respect to the true parameter.
The estimation error also increases polynomially with $\zeta$. 
By \Cref{def:sub-exponential},
increasing $\zeta$ implies the tail of sub-exponential distributions becomes increasingly
heavier than the tail of a Gaussian distribution.
Our result is similar to that of the sub-Gaussian distributions 
from \cite{ziemannTutorialNonAsymptoticTheory2024},
albeit with an extra $\log(T)$ factor.

The proof is detailed below,
and we provide a sketch here. 
The derivation involves finding the explicit form of $\epsilon$
using \Cref{lem:sub-exptail} and applying this result to \Cref{lem:main}
for the final sample complexity bound. 
In \Cref{lem:sub-exptail}, $\epsilon$ is obtained 
by applying the general concentration inequality 
\cite[corollary 6.20]{wainwrightHighDimensionalStatisticsNonAsymptotic2019} 
separately to $\Delta_1$ and $\Delta_2$ in \cref{eqn:emp-cov-decomp}. 
Since the corollary assumes bounded distributions
and sub-exponential distributions generally do not satisfy this property, 
we first ``clip'' the covariance so that it can be decomposed into a bounded part 
and an unbounded tail part, 
where the tail has a small probability due to the sub-exponential property
established in \Cref{prop:sub-exponential-tail}.

\begin{remark}
  Our result on sub-exponential distribution is generalizable to 
  sub-Weibull distributions addressed in \cite{Faradonbeh2017-ju}.
  Specifically, \cite[proposition 3]{Faradonbeh2017-ju} indicates that the tail of sub-Weibull distributions 
  is bounded exponentially.
  Adapting this bound in \Cref{thm:subexp} will achieve $\widetilde{\bigO}(r^{1/2}T^{-1/2})$, 
  which they arrived at for exponentially stable systems.
\end{remark}

We now discuss the formal proof. The structure will follow closely to \Cref{sec:thm-subgauss}.
First, we state a result that bounds the tail behavior of a sub-exponential distribution.
\begin{prop}\label{prop:sub-exponential-tail}
  An $r$-dimensional zero-mean $(\nu,\zeta)$-sub-exponential \gls*{rv} $X$ with $\nu,\zeta\in \mathbb{R}_{>0}$ satisfies the following:
  \begin{align}\label{eqn:sub-exponential-tail}
    \mathbb{P}[ \|X \|\geq t ] \leq 
    2r e^{-\min\left\{\frac{t}{2\zeta\sqrt{r}}, \frac{t^2}{2\nu^2r}\right\}}
  \end{align}
  \label{prop:subexpbound}
\end{prop}

\begin{proof}
  We observe that
  \begin{align*}
    \|X\| \geq t \Longrightarrow |X_i| \geq \frac{t}{\sqrt{r}} \text{ for some } i=1,\ldots,r.
  \end{align*}
  Therefore,
  \begin{align}\label{eqn:sub-exponential-tail-sum}
    \begin{aligned}[t]
      \mathbb{P}[\|X\| \geq t] 
      &\leq 
        \mathbb{P}\left[\bigcup_{i=1}^{r}
        \left(|X_i| \geq \frac{t}{\sqrt{r}}\right)
        \right]\\
      &\leq \sum_{i=1}^r \mathbb{P}\left[|X_i| \geq \frac{t}{\sqrt{r}}\right].
    \end{aligned}
  \end{align}
  Since $X$ is $(\nu,\zeta)$-sub-exponential, 
  it follows that the elements of $X$ are also $(\nu,\zeta)$-sub-exponential.
  By \cite[proposition 2.9]{wainwrightHighDimensionalStatisticsNonAsymptotic2019}, we have
  \begin{align}\label{eqn:sub-exponential-tail-1d}
    \mathbb{P}\left[|X_i| \geq \frac{t}{\sqrt{r}}\right] \leq 
    \begin{cases*}
      2e^{-t^2/2\nu^2 r} \text{ for } 0 \leq t \leq \nu^2\sqrt{r}/\zeta \\
      2e^{-t/2\zeta\sqrt{r}} \text{ for } t>\nu^2\sqrt{r}/\zeta.
    \end{cases*}    
  \end{align}
  By combining \cref{eqn:sub-exponential-tail-sum,eqn:sub-exponential-tail-1d}, 
  we obtain 
  \begin{align}
    \mathbb{P}[ \|X \|\geq t ] \leq 
    \begin{cases*}
      2r e^{-t^2 / 2\nu^2 r}, \quad 0 \leq t \leq \nu^2\sqrt{r}/\zeta \\
      2r e^{-t / 2\zeta\sqrt{r}}, \quad t>\nu^2\sqrt{r}/\zeta.
    \end{cases*}
  \end{align}
  which is equivalent to \cref{eqn:sub-exponential-tail}.
\end{proof}

\newcommand{\subexpepsilonone}{
  4k\zeta \log\left(\frac{4rkT}{\delta}\right)\log^{1/2}\left(\frac{8rk^2}{\delta}\right)r^{1/2}T^{1/2}
}
\newcommand{\subexpepsilontwo}{
  16k^2 \zeta^2 \log^2\left(\frac{4rkT}{\delta}\right)\log\left(\frac{8rk^2}{\delta}\right)r
}  

\newcommand{\subexpepsilon}{
  \subexpepsilonone + \subexpepsilontwo
}

Now we are ready to derive the tail bound of sub-exponential distribution 
and find $\epsilon$ explicitly in \Cref{lem:sub-exptail}. Our proof heavily relies on 
the concentration inequality for general matrices 
as stated in \cite[corollary 6.20]{wainwrightHighDimensionalStatisticsNonAsymptotic2019},
which requires bounded noise vectors.
This is where \cref{prop:subexpbound} comes in to 
bound the probability of large noise perturbation.
\begin{lem}[Tail Bound for Sub-Exponential \glspl*{rv}]\label{lem:sub-exptail}
  Under \Cref{assum:iid}, $k$ as defined in \cref{eqn:main-lem-assum-k}, 
  and $\delta\in(0,1)$, if $\xi_t$ is $(\nu,\zeta)$-sub-exponential for all $t=0,1,\cdots,T$,
  and
  \begin{align}
    T \geq \max\left\{ 
    \subexplembounds
    \right\}
    \label{eqn:subexp-assum}
  \end{align}
  the following holds with probability at least $1-\delta$:
  \begin{align*}
    \begin{aligned}[t]
      &\left\|\tail\right\|\\
      &\qquad\leq \subexpepsilon.
    \end{aligned}
  \end{align*}
\end{lem}

\begin{proof}
  Our goal is to apply the concentration inequality 
  provided by \cite[corollary 6.20]{wainwrightHighDimensionalStatisticsNonAsymptotic2019}
  which assumes that $\bxi_\tau^{(k,i)}$ is bounded.
  Our overarching idea of the proof is to differentiate between 
  bounded augmented noise vectors and the ones that are not. 
  For the bounded augmented noise vectors, we can directly apply the concentration inequality,
  and for the augmented noise vectors that are not bounded,
  we will show that the probability of this case is so small that it's becomes negligible
  using \Cref{prop:subexpbound}.

  Formally, for given $b>0$, we consider the following
  \begin{align*}
    \begin{aligned}
      \underline{\bxi}_{t:t-k+1} &= \bxi_{t:t-k+1} \cdot\mathds{1}\left\{\|\bxi_{t:t-k+1}\| \leq \sqrt{b}\right\},\\
      \underline{\bxi}_\tau^{(k,i)} &= \bxi_\tau^{(k,i)} \cdot\mathds{1}\left\{\|\bxi_\tau^{(k,i)}\| \leq \sqrt{b}\right\},
    \end{aligned}
    \quad
    \begin{aligned}
      \overline{\bxi}_{t:t-k+1} &= \bxi_{t:t-k+1} \cdot\mathds{1}\left\{\|\bxi_{t:t-k+1}\| > \sqrt{b}\right\}\\
      \overline{\bxi}_\tau^{(k,i)} &= \bxi_\tau^{(k,i)} \cdot\mathds{1}\left\{\|\bxi_\tau^{(k,i)}\| > \sqrt{b}\right\}.
    \end{aligned}
  \end{align*}
  We now can write
  \begin{align}\label{eqn:clipping}
    \begin{aligned}
      \bxi_{t:t-k+1}\bxi_{t:t-k+1}^\top &= \underline{\bxi}_{t:t-k+1}\underline{\bxi}_{t:t-k+1}^\top + \overline{\bxi}_{t:t-k+1}\overline{\bxi}_{t:t-k+1}^\top,\\
      \bxi_\tau^{(k,i)}(\bxi_\tau^{(k,i)})^\top &= \underline{\bxi}_\tau^{(k,i)}(\underline{\bxi}_\tau^{(k,i)})^\top + \overline{\bxi}_\tau^{(k,i)}(\overline{\bxi}_\tau^{(k,i)})^\top.
    \end{aligned}
  \end{align}

  The covariance of $\underline{\bxi}_\tau^{(k,i)}$ is bounded by the identity matrix. That is,
  \begin{align}\label{eqn:subexp-cov-bounded}
    \begin{aligned}
      \mathbb{E}\left[\underline{\bxi}_\tau^{(k,i)}(\underline{\bxi}_\tau^{(k,i)})^\top\right]
      &\preceq
        \mathbb{E}\left[\bxi_\tau^{(k,i)}(\bxi_\tau^{(k,i)})^\top\right]
        \stackrel{{\cref{assum:iid}}}{\preceq} I\\
    \end{aligned}
  \end{align}

  \paragraph{Bounding the individual tail probability}
  We now aim to bound the tail probability of $\bxi_{t:t-k+1}$.
  We fix
  \begin{align} 
    b = 4\zeta^2r k\log^2\left(\frac{4rkT}{\delta}\right),
    \label{eqn:subexp-b}
  \end{align}
  which satisfies
  \begin{align}
    \begin{aligned}[t]
      \left(\frac{\sqrt{b}}{2\zeta\sqrt{rk}} \right)/\left( \frac{b}{2\nu^2rk}\right)
      &
        \stackrel{\mathclap{\cref{eqn:subexp-b}}}{=}
        \log\left(\frac{4rkT}{\delta}\right) / \left( \frac{4\zeta^2r k}{2\nu^2rk} \log^2\left(\frac{4rkT}{\delta}\right)\right) \\
      &
        = 
        \frac{\nu^2}{2\zeta^2} \log^{-1}\left(\frac{4rkT}{\delta}\right) \\
      &
        \stackrel{\mathclap{\cref{eqn:subexp-assum}}}{\leq}
        1.
    \end{aligned}
    \label{eqn:subexp-b-observe}
  \end{align}
  From \Cref{prop:subexpbound}, we have
  \begin{align}
    \begin{aligned}[t]
      \mathbb{P}\left[ \left\|\bxi_{t:t-k+1} \right\|\geq \sqrt{b} \right]
      &\leq 2rk e^{-\min\left\{\frac{\sqrt{b}}{2\zeta\sqrt{rk}}, \frac{b}{2\nu^2rk}\right\}}\\
      &\stackrel{\mathclap{\cref{eqn:subexp-b-observe}}}{\leq} 2rk e^{-\frac{\sqrt{b}}{2\zeta\sqrt{rk}}}\\
      &\stackrel{\mathclap{\cref{eqn:subexp-b}}}{\leq} \frac{\delta}{2T}.
    \end{aligned}
    \label{eqn:subexp-prop-b-prob}
  \end{align}

  \paragraph{Bounds on $\Delta_{1,k}$}  
  We fix constant 
  \begin{align}
    \phi = \left(\frac{4b}{T^{(k,i)}}\log\left(\frac{8rk^2}{\delta}\right)\right)^{1/2}
    \label{eqn:subexp-prop-delta-def}
  \end{align}
  and observe the following fact:
  \begin{align}\label{eqn:subexp-assum-log}
    \log(aT) \leq T^{1/4}\log(a) \text{ for all } T \geq (6 + \sqrt{2})^8, a \geq e.
  \end{align}
  Since $4rk/\delta \geq e$ by \cref{eqn:main-lem-assum}, and $T \geq (6 + \sqrt{2})^8$ by \cref{eqn:subexp-assum},
  \begin{align}\label{eqn:subexp-assum-log-apply}
    \log\left(\frac{4rkT}{\delta}\right)
    \stackrel{\cref{eqn:subexp-assum-log},\cref{eqn:subexp-assum},\cref{eqn:main-lem-assum}}{\leq}
    T^{1/4}\log\left(\frac{4rk}{\delta}\right).
  \end{align}
  Therefore, $\phi$ in \cref{eqn:subexp-prop-delta-def} satisfies
  \begin{align}    \label{eqn:subexp-prop-delta-def-observe}
    \begin{aligned}[t]
      \phi^2
      &\stackrel{\cref{eqn:Tki-prac-bound}}{\leq} \frac{8kb}{T}\log\left(\frac{8rk^2}{\delta}\right)\\
      &\stackrel{{\cref{eqn:subexp-b}}}{=}\frac{32\zeta^2k^2 r}{T} \log^2 \left(\frac{4rkT}{\delta}\right)\log\left(\frac{8rk^2}{\delta}\right)\\
      &\stackrel{\cref{eqn:subexp-assum-log-apply}}{\leq} \frac{32 \zeta^2 k^2 r}{T^{1/2} } \log^3\left(\frac{8rk^2}{\delta}\right)\\
      &\stackrel{\cref{eqn:subexp-assum}}{\leq} 1.
    \end{aligned}
  \end{align}  
  Since $\|\underline{\bxi}_\tau^{(k,i)}\|$ is upper bounded, 
  one can apply \cite[corollary 6.20]{wainwrightHighDimensionalStatisticsNonAsymptotic2019} 
  and bounded covariance property in \cref{eqn:subexp-cov-bounded} to obtain
  \begin{align}\label{eqn:subexp-prop-term1-bound}
    \begin{aligned}[t]
      \left\|\termonemonounder\right\|
      &\;\;\leq\;\; \phi T^{(k,i)}\\
      &
        \stackrel{\mathclap{\cref{eqn:subexp-prop-delta-def},\cref{eqn:Tki-bound}}}{\;\;\leq\;\;}
        \left(\frac{4bT}{k}\log\left(\frac{8rk^2}{\delta}\right)\right)^{1/2},
    \end{aligned}
  \end{align}
  which holds with probability at least
  \begin{align}\label{eqn:subex-term1-tail-prob}
    \begin{aligned}[t]
      1-2rke^{-\frac{T^{(k,i)}\phi^2}{2b(1+\phi)}}
      &\stackrel{\mathclap{\cref{eqn:subexp-prop-delta-def-observe}}}{\geq}
        1-2rke^{-\frac{T^{(k,i)}\phi^2}{4b}}\\
      &\stackrel{\mathclap{\cref{eqn:subexp-prop-delta-def}}}{\geq}
        1 - \frac{\delta}{4k}.
    \end{aligned}
  \end{align}

  \paragraph{Bounds on $\Delta_{2,k}$}
  We fix
  \begin{align}\label{eqn:subexp-phi-prime}
    \phi' = 4b\log\left(\frac{8rk^2}{\delta}\right).
  \end{align}
  and observe that
  \begin{align}\label{eqn:subexp-phi-prime-observe}
    \begin{aligned}[t]
      \phi' 
      &\stackrel{\mathclap{\cref{eqn:subexp-b}}}{=} 16\zeta^2k r \log^2 \left(\frac{4rkT}{\delta}\right)\log\left(\frac{8rk^2}{\delta}\right)\\
      &\stackrel{\mathclap{\cref{eqn:subexp-assum}}}{\geq} kr \log\left(\frac{8rk^2}{\delta}\right)\\
      &\geq 1,
    \end{aligned}
  \end{align}
  where the last inequality follows from $k,r\geq 1$ and $\delta\in(0,1)$.
  By applying \cite[corollary 6.20]{wainwrightHighDimensionalStatisticsNonAsymptotic2019} and bounded covariance property in \cref{eqn:subexp-cov-bounded}, we obtain
  \begin{align} \label{eqn:subexp-prop-term2-bound}
    \begin{aligned}[t]
      \left\|\termtwomonounder \right\| 
      &\leq \phi'\\ 
      &
        \stackrel{\mathclap{\cref{eqn:subexp-phi-prime}}}{=}
        4b\log\left(\frac{8rk^2}{\delta}\right)
    \end{aligned}
  \end{align}
  with probability at least 
  \begin{align}\label{eqn:subex-term2-tail-prob}
    \begin{aligned}
      1- 2rk e^{-\frac{\phi'^2}{2b(1+\phi')}}
      &\stackrel{\mathclap{\cref{eqn:subexp-phi-prime-observe}}}{\geq}
      1- 2rk e^{-\frac{\phi'}{4b}}\\
      &\stackrel{\mathclap{\cref{eqn:subexp-phi-prime}}}{\geq}
      1 - \frac{\delta}{4k}.
    \end{aligned}
  \end{align}
  
  \paragraph{Combining terms together}
  Combining $\Delta_{1,k}$ and $\Delta_{2,k}$ and using the triangle inequality,
  we obtain
  \begin{align}\label{eqn:subExp-bound}
    \begin{aligned}[t]
      &\left\|\tail\right\|\\
      &\stackrel{\cref{eqn:clipping}}{\qquad\leq\qquad}
        \left\|\sum_{t=1}^T \overline{\bxi}_{t:t-k+1} \overline{\bxi}_{t:t-k+1}^\top - \mathbb{E}\left[\sum_{t=1}^T \overline{\bxi}_{t:t-k+1} \overline{\bxi}_{t:t-k+1}^\top\right]\right\| \\
      &\qquad\qquad+
        \left\|\sum_{i=0}^{k-1} \left(
        \sum_{\tau=1}^{T^{(k,i)}} \underline{\bxi}^{(k,i)}_\tau ( \underline{\bxi}^{(k,i)}_\tau)^\top
        -\mathbb{E}\left[\sum_{\tau=1}^{T^{(k,i)}} \underline{\bxi}^{(k,i)}_\tau ( \underline{\bxi}^{(k,i)}_\tau)^\top\right]
          \right)\right\| \\
      &\qquad\qquad+
        \left\|\sum_{i=0}^{k-1}  \left(\underline{\bxi}^{(k,i)}_{0}(\underline{\bxi}^{(k,i)}_{0})^\top - \mathbb{E}\left[\underline{\bxi}^{(k,i)}_{0}(\underline{\bxi}^{(k,i)}_{0})^\top\right] \right)
        \right\|      
      \\
      & \stackrel{\mathclap{\cref{eqn:subexp-prop-b-prob},\cref{eqn:subexp-prop-term1-bound},\cref{eqn:subexp-prop-term2-bound}}}{\qquad\leq\qquad}
        \sum_{i=0}^{k-1}\left(\frac{4bT}{k}\log\left(\frac{8rk^2}{\delta}\right)\right)^{1/2} + \sum_{i=0}^{k-1} 4b\log\left(\frac{8rk^2}{\delta}\right)\\
      &\stackrel{\cref{eqn:subexp-b}}{\qquad\leq\qquad}
        \subexpepsilonone+\subexpepsilontwo \\
    \end{aligned}
  \end{align}
  with probability at least 
  \begin{align*}
    1- \sum_{t=1}^T \frac{\delta}{2T} -\sum_{i=0}^{k-1} \frac{\delta}{4k} - \sum_{i=0}^{k-1} \frac{\delta}{4k} = 1-\delta.
  \end{align*}
  due to \cref{eqn:subexp-prop-b-prob,eqn:subex-term1-tail-prob,eqn:subex-term2-tail-prob}.
  Here, $\delta/2T$ bounds the probability that $\overline{\bxi}_{t:t-k+1}$ exceeds the bound $\sqrt{b}$,
  which allows us to apply the concentration inequality, and $\delta/4k$ bounds the deviation of empirical covariance
  from the covariance.

\end{proof}
We are now ready to state the sample complexity bound of sub-exponential distribution
by applying \Cref{lem:main}.
{
  \def\proofname{Proof of \Cref{thm:subexp}} 
  \begin{proof} 
    Suppose that we choose $k$ as in \cref{eqn:main-lem-assum-k} and that $T$ satisfies
    \begin{align}
      \label{eqn:subexp-thm-assum}
      T \geq 
      \max
      \Bigg\{
      & 
        \subexplembounds,
        2\overline{t}_c,\\
      &
        \subexpthmboundone,
        \subexpthmboundtwo
        \Bigg\}.\nonumber
    \end{align}
    To provide an upper bound on $\|\widehat{\Theta} - \Theta\|$,
    we will apply $\epsilon$ from \Cref{lem:sub-exptail}.
    Since the conditions of \Cref{lem:sub-exptail}, specifically \cref{eqn:subexp-assum},
    are satisfied by \cref{eqn:subexp-thm-assum},
    we just have to show that the proposed $\epsilon$ from \Cref{lem:sub-exptail} 
    satisfies \cref{eqn:main-lem-assum-eps}:
    \begin{align}\label{eqn:subexp-eps-check-1}
      \begin{aligned}[t]
        \subexpepsilonone&\stackrel{\mathclap{\cref{eqn:subexp-assum-log-apply}}}{\leq}
                           4 \zeta k \log^{3/2}\left(\frac{8rk^2}{\delta}\right) r^{1/2} T^{3/4}
        \\
                         &\leq
                           \frac{1}{2}\cdot\epsbound,
      \end{aligned}
    \end{align}
    where the last inequality follows from
    \begin{align*}
      \subexpthmboundone
      \stackrel{\mathclap{\cref{eqn:subexp-thm-assum}}}{\leq}
      T.
    \end{align*}
    Furthermore,
    \begin{align}\label{eqn:subexp-eps-check-2}
      \begin{aligned}[t]
        \subexpepsilontwo
        &\stackrel{\mathclap{\cref{eqn:subexp-assum-log-apply}}}{\leq}
          16k^2\zeta^2 \log^3\left(\frac{8rk^2}{\delta}\right) rT^{1/2}
        \\
        &\leq
          \frac{1}{2}\cdot\epsbound,
      \end{aligned}
    \end{align}
    where the last inequality follows from
    \begin{align*}
      \subexpthmboundtwo
      \stackrel{\mathclap{\cref{eqn:subexp-thm-assum}}}{\leq}
      T.
    \end{align*}
    \Cref{lem:sub-exptail}, \cref{eqn:subexp-eps-check-1}, and \cref{eqn:subexp-eps-check-2} indicate that \Cref{assum:tail} is satisfied with $\epsilon$ satisfying \cref{eqn:main-lem-assum-eps}.
    Finally, \cref{eqn:main-lem-assum-T} is satisfied due to \cref{eqn:subexp-thm-assum}, confirming
    that all conditions in \cref{eqn:main-lem-assum} are satisfied.
    Therefore, we may apply \Cref{lem:main}, which results in
    \begin{align}\label{eqn:thm-1-final-result}
      \begin{aligned}[t]
        \left\|\widehat{\Theta} - \Theta\right\|
        &\leq \frac{ 8 \sigma^2 L}{ (1-\alpha) \beta^2 T}\subexpepsilonone\\
        &\qquad+\frac{ 8 \sigma^2 L}{ (1-\alpha) \beta^2 T}\subexpepsilontwo 
          + \frac{8\sigma^2 L}{ (1-\alpha) \beta^2 } T^{-1/2}\\
        &\stackrel{\cref{eqn:subexp-thm-assum}}{\leq} \subexpfinalbound.
      \end{aligned}
    \end{align}
    By inspecting \cref{eqn:subexp-thm-assum,eqn:thm-1-final-result}, 
    one can see that there exist universal polynomials $\polyEburn(\cdot),\polyEbound(\cdot)$ 
    such that \cref{eqn:subexp-thm-assum,eqn:thm-1-final-result} hold under \cref{eqn:subexp-thm-assum-poly,eqn:subexp-thm}.
    Specifically, the terms in $\polyEburn(\cdot)$ comes from breaking down the terms in \cref{eqn:subexp-thm-assum}
    and replacing $k$ with its polynomial equivalence from the definition \cref{eqn:main-lem-assum-k}. 
    We note that $k$ depends polynomially on $\log(T)$. Since $\log(T)$ grows more slowly than any positive power of $T$, this dependence can be absorbed into a polynomial dependence on $T$ for sufficiently large $T$.
    $\polyEbound(\cdot)$ in \cref{eqn:subexp-thm} can also be 
    derived from \cref{eqn:thm-1-final-result} using the same approach,
    thus completing the proof.
  \end{proof}
}

\section{Proof of \Cref{thm:pthmom}}\label{sec:thm-pthmom}
The proof structure will follow closely to \Cref{sec:thm-subexp}. 
First, we state a result that bounds the tail decay of a noise vector
with bounded $p$th moment.
\begin{prop} \label{prop:pthmom-indcov}
  If $\xi_t$ for all $t=0,1,\cdots,T$ satisfies
  \begin{align}
    \mathbb{E}[\|\xi_t\|^p] \leq \kappa r^{p/2},
  \end{align}
  the following holds:
  \begin{align}
    \mathbb{P}[\|\bxi_{\tau}^{(k,i)}\| \geq b^{1/2}] \leq \kappa r^{p/2} k^{p/2+1}b^{-p/2}.
  \end{align}
\end{prop}

\begin{proof}
  To satisfy $\| \bxi_\tau^{(k,i)}\| \geq b^{1/2}$ for $b \geq 0$,
  there exists $j \in[(\tau-1) k + i + 2, \tau k + i + 1]$ such that 
  $\| \xi_{j} \| \geq (b/k)^{1/2}$.
  Applying Markov Inequality to \cref{eqn:pthmom-assum-moment}, we have
  \begin{align*}
    \mathbb{P}[\|\xi_{t}\| \geq (b/k)^{1/2}] \leq {\kappa r^{p/2}}{(b/k)^{-p/2}}
  \end{align*}
  for all $t=0,1,\cdots,T$.
  Therefore,
  \begin{align*}
    \mathbb{P}\left[\left\| \bxi_\tau^{(k,i)}\right\| \geq b^{1/2}\right] 
    &\leq \mathbb{P} \left[\bigcup_{j=(\tau-1) k + i + 2}^{ \tau k + i + 1} \left\|\xi_{j} \right\| \geq (b/k)^{1/2}\right] \\
    &\leq \sum_{j=(\tau-1) k + i + 2}^{ \tau k + i + 1}\mathbb{P} \left[\left\|\xi_{j} \right\| \geq (b/k)^{1/2}\right] \\
    &\leq k\left({\kappa r^{p/2} k^{p/2}}{b^{-p/2}} \right),
  \end{align*}
  which completes the proof.
\end{proof}

We can then derive the tail bound of distributions with bounded $p$th moment 
using \cite[corollary 6.20]{wainwrightHighDimensionalStatisticsNonAsymptotic2019} 
and \Cref{prop:pthmom-indcov}.
\begin{lem}[Tail Bound for \glspl*{rv} with Bounded $p$th Moment]\label{lem:pthmom}
  Under \Cref{assum:iid}, 
  $k$ as defined in \cref{eqn:main-lem-assum-k}, 
  and $\delta\in(0,1)$, if $\xi_t$ is $\kappa$-bounded $p$th moment with $p>2$ and
  \begin{align}
    \begin{aligned}[t]
      T \geq \max\Bigg\{
      &4k,
        \pthmomboundone,\\
      &\pthmomboundtwo
        \Bigg\},
    \end{aligned}
    \label{eqn:pthmom-assum}
  \end{align}
  the following holds with probability at least $1-\delta$:
  \begin{align*}
    &\left\|\tail\right\|\\
    &\leq \pthmomepsilon.
  \end{align*}
\end{lem}

\begin{proof}
  Similar to sub-exponential distributions, 
  we aim to apply \cite[corollary 6.20]{wainwrightHighDimensionalStatisticsNonAsymptotic2019}, 
  so we use \Cref{prop:pthmom-indcov} to bound the tail probability
  of a distribution with bounded $p$th moment. 
  Let 
  \begin{align} 
    b = \left(\frac{2T\kappa r^{p/2} k^{p/2+1}}{\delta}\right)^{2/p}.
    \label{eqn:pthmom-b-def}
  \end{align}
  Due to \Cref{prop:pthmom-indcov}, we have that
  \begin{align}\label{eqn:pthmom-prop-prob-bound-0}
    \|\bxi_\tau^{(k,i)} \| \leq b^{1/2}
  \end{align}
  with probability at least
  \begin{align}\label{eqn:pthmom-prop-prob-bound-1}
    1 - \frac{\kappa r^{p/2} k^{p/2+1}}{b^{p/2}} = 1 - \frac{\delta}{2T}.
  \end{align}
  
  \paragraph{Bounding $\Delta_{1,k}$} 
  We fix
  \begin{align}\label{eqn:pthmom-phi-def}
    \phi = \left(\frac{4b}{T^{(k,i)}}\log\left(\frac{8rk^2}{\delta}\right)\right)^{1/2}.
  \end{align}
  and observe that
    \begin{align}\label{eqn:pthmom-phi-observe}
      \begin{aligned}[t]
        \phi^2
        &\stackrel{\mathclap{\cref{eqn:pthmom-b-def}}}{=}
          \left(\frac{2T\kappa r^{p/2} k^{p/2+1}}{\delta}\right)^{2/p}\frac{4}{T^{(k,i)}}\log\left(\frac{8rk^2}{\delta}\right)\\
        &\stackrel{\mathclap{\cref{eqn:Tki-prac-bound}}}{\leq}
          \left(\frac{2T\kappa r^{p/2} k^{p/2+1}}{\delta}\right)^{2/p}\frac{8k}{T}\log\left(\frac{8rk^2}{\delta}\right)\\
        &\stackrel{\mathclap{\cref{eqn:pthmom-assum}}}{\leq}
          1.
      \end{aligned}
  \end{align}  
  We now apply \cite[corollary 6.20]{wainwrightHighDimensionalStatisticsNonAsymptotic2019} 
  and the bounded covariance condition \cref{eqn:subexp-cov-bounded} to obtain
  \begin{align}\label{eqn:pthmom-prop-term1-bound}
    \begin{aligned}[t]
      \left\|\termonemonounder\right\|
      &\leq \phi T^{(k,i)}\\
      & \stackrel{\cref{eqn:pthmom-phi-def},\cref{eqn:Tki-bound}}{\leq} \left(\frac{4bT}{k}\log\left(\frac{8rk^2}{\delta}\right)\right)^{1/2}
    \end{aligned}
  \end{align}
  with probability at least 
  \begin{align} \label{eqn:pthmom-prop-prob-bound-2}
    \begin{aligned}
      1- 2rke^{-\frac{T^{(k,i)}\phi^2}{2b(1+\phi)}}
      &\stackrel{\cref{eqn:pthmom-phi-observe}}{\geq} 1- 2rke^{-\frac{T^{(k,i)}\phi^2}{4b}}\\
      &\stackrel{\cref{eqn:pthmom-phi-def}}{=} 1 - \frac{\delta}{4k}
    \end{aligned}
  \end{align}

  
  \paragraph{Bounding $\Delta_{2,k}$}
  We fix
  \begin{align}\label{eqn:pthmom-phi-prime-def}
    \phi' = 4b\log\left(\frac{8rk^2}{\delta}\right).
  \end{align}
  and observe
  \begin{align}\label{eqn:pthmom-phi-prime-observe}
    \begin{aligned}[t]
      \phi'
      &\stackrel{\mathclap{\cref{eqn:pthmom-b-def}}}{=}
        4\left(\frac{2\kappa r^{p/2} k^{p/2+1}}{\delta}\right)^{2/p}\log\left(\frac{8rk^2}{\delta}\right) T^{2/p}\\
      &\stackrel{\mathclap{\cref{eqn:pthmom-assum}}}{\geq}
        1.
    \end{aligned}
  \end{align}
  Similarly to $\Delta_{1,k}$, we have
  \begin{align}
    \left\| \termtwomonounder \right\| 
    \leq \phi' = 4b\log\left(\frac{8rk^2}{\delta}\right)
    \label{eqn:pthmom-prop-term2-bound}
  \end{align}
  with probability at least 
  \begin{align}\label{eqn:pthmom-prop-prob-bound-3} 
    1- 2rke^{-\frac{\phi'^2}{2b(1+\phi')}} \geq  1- 2rke^{-\frac{\phi'}{4b}} \geq  1 - \frac{\delta}{4k},
  \end{align}
  \paragraph{Combining both terms}

  Combining $\Delta_{1,k}$ and $\Delta_{2,k}$ and using the triangle inequality,
  we obtain
  \begin{align}\label{eqn:subexp-bound}
    \begin{aligned}[t]
      &\left\|\tail\right\|\\
      &\stackrel{\cref{eqn:clipping}}{\qquad\leq\qquad}
        \left\|\sum_{t=1}^T \overline{\bxi}_{t:t-k+1} \overline{\bxi}_{t:t-k+1}^\top - \mathbb{E}\left[\sum_{t=1}^T \overline{\bxi}_{t:t-k+1} \overline{\bxi}_{t:t-k+1}^\top\right]\right\| \\
      &\qquad\qquad+
        \left\|\sum_{i=0}^{k-1} \left(
        \sum_{\tau=1}^{T^{(k,i)}} \underline{\bxi}^{(k,i)}_\tau ( \underline{\bxi}^{(k,i)}_\tau)^\top
        -\mathbb{E}\left[\sum_{\tau=1}^{T^{(k,i)}} \underline{\bxi}^{(k,i)}_\tau ( \underline{\bxi}^{(k,i)}_\tau)^\top\right]
          \right)\right\| \\
      &\qquad\qquad+
        \left\|\sum_{i=0}^{k-1}  \left(\underline{\bxi}^{(k,i)}_{0}(\underline{\bxi}^{(k,i)}_{0})^\top - \mathbb{E}\left[\underline{\bxi}^{(k,i)}_{0}(\underline{\bxi}^{(k,i)}_{0})^\top\right] \right)
        \right\|      
      \\
      & \stackrel{\mathclap{\cref{eqn:pthmom-prop-prob-bound-0},\cref{eqn:pthmom-prop-term1-bound},\cref{eqn:pthmom-prop-term2-bound}}}{\qquad\leq\qquad}
        \sum_{i=0}^{k-1}\left(\frac{4bT}{k}\log\left(\frac{8rk^2}{\delta}\right)\right)^{1/2} + \sum_{i=0}^{k-1} 4b\log\left(\frac{8rk^2}{\delta}\right)\\
      &\qquad\leq\qquad
        \left(4kbT\log\left(\frac{8rk^2}{\delta}\right)\right)^{1/2} + 4kb\log\left(\frac{8rk^2}{\delta}\right)\\
      &\qquad\leq\qquad
        \pthmomepsilon.
    \end{aligned}
  \end{align}
  with probability at least 
  \begin{align*}
    1- \sum_{t=1}^T \frac{\delta}{2T} -\sum_{i=0}^{k-1} \frac{\delta}{4k} - \sum_{i=0}^{k-1} \frac{\delta}{4k} = 1-\delta.
  \end{align*}
  due to \cref{eqn:pthmom-prop-prob-bound-1,eqn:pthmom-prop-prob-bound-2,eqn:pthmom-prop-prob-bound-3}.

\end{proof}

Finally, we apply \Cref{lem:main} by using the explicit form of $\epsilon$ 
derived in \Cref{lem:pthmom} to conclude the sample complexity bound on
distributions with bounded $p$th moment.
{
\def\proofname{Proof of \Cref{thm:pthmom}} 
\begin{proof}
  Suppose that we choose $k$ as in \cref{eqn:main-lem-assum-k}  and  that the following inequality hold:
  \begin{align} 
    \begin{aligned}[t]
      T \geq \max\Bigg\{
      &
        r^{p/(p-2)}, 4k, \pthmomboundone,\\
      &
        \pthmomboundtwo, 2\overline{t}_c, \\
      &
        \pthmomthmboundone, \\
      &
        \pthmomthmboundtwo
        \Bigg\}.
    \end{aligned}
    \label{eqn:pthmom-thm-assum}
  \end{align}
  To provide an upper bound on $\|\widehat{\Theta} - \Theta\|$,
  we will apply the proposed $\epsilon$ derived in \Cref{lem:pthmom}.
  Now we proceed to prove the theorem
  by assuming \cref{eqn:pthmom-thm-assum} holds.
  Since the  conditions of \Cref{lem:pthmom} (which are \cref{eqn:pthmom-assum}) 
  are satisfied by \cref{eqn:pthmom-thm-assum},
  then to use \Cref{lem:main}, we have to show \cref{eqn:main-lem-assum-eps} holds.
  We have:
    \begin{align} \label{eqn:thm-pthmom-assum1}
      \pthmomepsilonone
      &\stackrel{\mathclap{\cref{eqn:pthmom-thm-assum}}}{\leq}
      \frac{1}{2}\cdot\epsbound,
    \end{align}
    and
    \begin{align} \label{eqn:thm-pthmom-assum2}
      \pthmomepsilontwo
      &\stackrel{\mathclap{\cref{eqn:pthmom-thm-assum}}}{\leq}
        \frac{1}{2}\cdot\epsbound.
    \end{align}
    \Cref{eqn:thm-pthmom-assum1,eqn:thm-pthmom-assum2} indicate that 
    $\epsilon$ from \Cref{lem:pthmom} satisfies \cref{eqn:main-lem-assum-eps},
    so \Cref{assum:tail} holds.
    Finally, \cref{eqn:main-lem-assum-T} is satisfied due to \cref{eqn:pthmom-thm-assum}, confirming
    that all conditions in \cref{eqn:main-lem-assum} are satisfied.
    Therefore, we may apply \Cref{lem:main}, which indicates that
    \begin{align}\label{eqn:thm-2-final-result}
      \begin{aligned}[t]
        &\left\|\widehat{\Theta} - \Theta\right\|\\
        &\leq \frac{ 8 \sigma^2 L}{ (1-\alpha) \beta^2 T}\left( \pthmomepsilon\right) \\
        &\qquad + \frac{8\sigma^2 L}{ (1-\alpha) \beta^2 } T^{-1/2}\\
        &\leq \pthmomfinalbound,
      \end{aligned}
    \end{align}
  where the last inequality follows from  $r^{1/2}T^{1/2+1/p} \geq rT^{2/p}$. 
  By inspecting \cref{eqn:pthmom-thm-assum,eqn:thm-2-final-result}, 
  one can see that there exist $p$-dependent polynomials $\polyPburn(\cdot),\polyPbound(\cdot)$ 
  such that \cref{eqn:pthmom-thm-assum,eqn:thm-2-final-result} hold 
  under \cref{eqn:pthmom-thm-assum-poly,eqn:pthmom-thm}.
  Specifically, 
  the terms in $\polyPburn(\cdot)$ comes from breaking down the terms in \cref{eqn:pthmom-thm-assum}
  and replacing $k$ with its polynomial equivalence from the definition \cref{eqn:main-lem-assum-k}. 
  We note that $k$ depends polynomially on $\log(T)$. Since $\log(T)$ grows more slowly than any positive power of $T$, this dependence can be absorbed into a polynomial dependence on $T$ for sufficiently large $T$.
  $\polyPbound(\cdot)$ in \cref{eqn:pthmom-thm} can also be 
  derived from \cref{eqn:thm-2-final-result} using the same approach,
  thus completing the proof.
\end{proof}




}

\section{Rewriting \gls*{arx} Model as \gls*{var} Model}
\label{sec:arx-appendix}
Assume that the input sequence $\{v_t\}$ is a series of independent \glspl*{rv} such as $v_t \sim \mathcal{N}(0,1)$.
We now observe that \cref{eqn:arx} can be rewritten as follows for $t=1,2,\cdots,T$:
\begin{align}\label{eqn:arx-augmented}
  \underbrace{\begin{bmatrix}
    y_{t} \\
    y_{t-1} \\
    \vdots    \\
    y_{t-q+1}     \\
    u_{t} \\
    u_{t-1} \\
    \vdots    \\
    u_{t-q+1}     \\
  \end{bmatrix}}_{z_{t}}
  =
  \underbrace{\begin{bmatrix}
    A_1    & \cdots & A_{q-1} & A_q    & B_1    & \cdots & B_{q-1} & B_q    \\
    I      & \cdots & 0       & 0      & 0      & \cdots & 0       & 0      \\
    \vdots & \ddots & \vdots  & \vdots & \vdots & \ddots & \vdots  & \vdots \\
    0      & \cdots & I       & 0      & 0      & \cdots & 0       & 0      \\
    0      & \cdots & 0       & 0      & 0      & \cdots & 0       & 0      \\
    0      & \cdots & 0       & 0      & I      & \cdots & 0       & 0      \\
    \vdots & \ddots & \vdots  & \vdots & \vdots & \ddots & \vdots  & \vdots \\
    0      & \cdots & 0       & 0      & 0      & \cdots & I       & 0
  \end{bmatrix}}_{\Theta}
  \underbrace{\begin{bmatrix}
    y_{t-1}   \\
    y_{t-2}   \\
    \vdots    \\
    y_{t-q}   \\
    u_{t-1}   \\
    u_{t-2} \\
    \vdots    \\
    u_{t-q}   \\
  \end{bmatrix}}_{z_{t-1}}
  +
  \underbrace{\begin{bmatrix}
    F & 0          \\
    0 & 0          \\
    \vdots & \vdots     \\
    0 & 0          \\
    0 & G        \\
    0 & 0          \\
    \vdots & \vdots     \\
    0 & 0          \\
  \end{bmatrix}}_{\Psi}
  \underbrace{\begin{bmatrix}
    w_t \\
    v_t        \\
  \end{bmatrix}}_{\xi_t}. 
\end{align}

\section{Proof of \Cref{thm:arx}}\label{sec:thm-arx}
\subsection{Verifying \Cref{assum:iid}}
Follows directly from \Cref{assum:arx-iid}.
\subsection{Verifying \Cref{assum:stab}}
From the explicit form of $\Theta$ \cref{eqn:arx-augmented},
we observe that $\Theta$ can be partitioned as follows:
\begin{align*}
  \Theta &\stackrel{\mathclap{\cref{eqn:arx-ABCD-def}}}{=}
         \begin{bmatrix}
           \bA & \bB \\
           0 & \bE
         \end{bmatrix},
\end{align*}
where $\bE$ is a nilpotent matrix with order $q$ and $\|\bE\| = 1$. 
We now observe that
\begin{align}\label{eqn:phi-t}
  \Theta^t &=
           \begin{bmatrix}
             \bA^t & \sum_{i=0}^{t-1} \bA^{t-1-i} \bB \bE^{i} \\
             0   & \bE^t
           \end{bmatrix}.
\end{align}
Furthermore, the stability assumption in \Cref{assum:arx-main} states
\begin{align}\label{eqn:arx-proof-1}
  \|\bA^t\| \stackrel{\mathclap{\cref{assum:arx-main}}}{\;\;\leq\;\;} c_A \gamma^t, 
  \quad \forall t=0,1,\cdots;
\end{align}
By the nilpotent property of $\bE$, we have that
\begin{align}\label{eqn:arx-proof-2}
  \|\bE^t\| \leq \mathds{1}\{t< q\}, \quad \forall t=0,1,\cdots;
\end{align}
Then by stability of $\bA$ and nilpotent property of $\bE$,
\begin{align}\label{eqn:arx-proof-3}
  \begin{aligned}
  \left\|\sum_{i=0}^{t-1} \bA^{t-1-i} \bB \bE^{i} \right\| 
  &\stackrel{\mathclap{\cref{assum:arx-main}}}{\quad\leq\quad\quad} 
    c_B \sum_{i=0}^{t-1} \left\|\bA^{t-1-i}\right\| \left\| \bE^{i} \right\| \\
  &\stackrel{\mathclap{\cref{eqn:arx-proof-1},\cref{eqn:arx-proof-2}}}{\quad\leq\quad\quad} 
    c_A c_B\sum_{i=0}^{q-1} \gamma^{t-1-i}  \\
  &\quad\leq\quad\quad c_A c_B \gamma^{t-1}\frac{\gamma^{-q}-1}{\gamma^{-1}-1}.
  \end{aligned}
\end{align}
By the block stucture of $\Theta$ in \cref{eqn:phi-t}, we have that:
\begin{align*}
  \|\Theta^t\|
  &\quad\leq\quad \max(\|\bA^t\|,\|\bE^t\|) + \left\|\sum_{i=0}^{t-1} \bA^{t-1-i} \bB \bE^{i} \right\|\\
  &\stackrel{\mathclap{\cref{eqn:arx-proof-1},\cref{eqn:arx-proof-2},\cref{eqn:arx-proof-3}}}{\quad\leq\quad} 
    \max(c_A \gamma^t, \mathds{1}\{t< q\}) + c_A c_B \gamma^{t-1}\frac{\gamma^{-q}-1}{\gamma^{-1}-1}\\ 
  &\quad\leq\quad \left(\max(c_A, \gamma^{-q}) + c_A c_B \gamma^{-1}\frac{\gamma^{-q}-1}{\gamma^{-1}-1} \right)\gamma^{t},
\end{align*}
where the last inequality follows from $\gamma \in (0,1)$.
\subsection{Verifying \Cref{assum:bounded-cov}}
Directly follows from the structure of $\Psi$.
\subsection{Verifying \Cref{assum:ctrb-sigma}}
Let $\overline{t}_c = q$.
Consider the linear map defined by the matrix
$\begin{bmatrix}\Psi & \Theta\Psi & \cdots & \Theta^{q-1}\Psi\end{bmatrix}$ in \cref{eqn:ctrb-sigma},
which maps $(\xi_q,\ldots,\xi_1)$ to the state $z_q$ reached at time $q$
by the recursion in \cref{eqn:arx-augmented} started from the zero state:
\begin{align} \label{eqn:arx-sys}
  \begin{aligned}
    \begin{bmatrix}
    y_q\\
    \vdots \\
    y_{1}\\
    u_q\\
    \vdots \\
    u_{1}\\
  \end{bmatrix}
  \stackrel{\mathclap{\cref{eqn:arx-augmented}}}{\;\;=\;\;} 
  \Psi 
  \begin{bmatrix}
    w_{q}\\
    v_{q}
  \end{bmatrix}
  + 
  \Theta \Psi
  \begin{bmatrix}
    w_{q-1}\\
    v_{q-1}
  \end{bmatrix}
  + \cdots 
  + \Theta^{q-1} \Psi
  \begin{bmatrix}
    w_{1}\\
    v_{1}
  \end{bmatrix},
  \end{aligned}
\end{align}
where $\Theta, \Psi$ are also defined in \cref{eqn:arx-augmented}.
For any $z_q$ such that $\|z_q\| = 1$,
we observe the following is a solution to \cref{eqn:arx-sys}:
\begin{align}
  \begin{aligned}
    w_{t} = F^{-1} \left(y_{t} - \sum_{i=1}^{t-1} (A_i y_{t-i} + B_i u_{t-i})\right), \quad
    v_t = G^{-1} u_{t}, \quad
    t = 1, \cdots, q
  \end{aligned}
\end{align}
We note that $F$ and $G$ are invertible due to \Cref{assum:arx-noise}.
Explicitly, we can write the solution to \cref{eqn:arx-sys} as
\begin{align}\label{eqn:M-def}
  \begin{aligned}[t]
    \begin{bmatrix}
      w_{q}\\
      \vdots \\
      w_{1}\\
      v_{q}\\
      \vdots \\
      v_{1}\\
    \end{bmatrix}
    &= 
    \underbrace{
    \left[
    \begin{array}{ccc;{2pt/2pt}ccc}
      F^{-1} &        &        &        &        & \\
             & \ddots &        &        &        & \\
             &        & F^{-1} &        &        & \\ \hdashline[2pt/2pt]
             &        &        & G^{-1} &        & \\
             &        &        &        & \ddots & \\
             &        &        &        &        & G^{-1}
    \end{array}
    \right]
    \left[
    \begin{array}{c;{2pt/2pt}c}
      \bC & \bD \\ \hdashline[2pt/2pt]
      0 & I
    \end{array}
    \right]
    }_{M}
      \begin{bmatrix}
        y_q\\
        \vdots \\
        y_{1}\\
        u_q\\
        \vdots \\
        u_{1}\\
      \end{bmatrix}
  \end{aligned}
\end{align}
where $\bC,\bD$ are defined in \cref{eqn:arx-ABCD-def}. 
Since $\begin{bmatrix}\Psi & \Theta \Psi & \cdots & \Theta^{q-1} \Psi\end{bmatrix}$
is an $\mathbb{R}^{(d_y+d_u)q\times(d_y+d_u)q}$ square matrix, $M$ is its inverse (up to permutation of rows and columns).  Thus,
\begin{align*}
  \begin{bmatrix}\Psi & \Theta \Psi & \cdots & \Theta^{q-1} \Psi\end{bmatrix} \begin{bmatrix}\Psi & \Theta \Psi & \cdots & \Theta^{q-1} \Psi\end{bmatrix}^\top
  &\;\;\succeq\;\; \frac{1}{\|M\|^2} I\\ 
  &\stackrel{\mathclap{\cref{eqn:M-def}}}{\;\;\succeq\;\;} \frac{1}{\max(\|F^{-1}\|,\|G^{-1}\|)^2(\max(1, \|\bC\|) + \|\bD\|)^2} I\\
  &\stackrel{\mathclap{\cref{assum:arx-main}}}{\;\;\succeq\;\;} \frac{\min(\underline{\sigma}_F,\underline{\sigma}_G)^2}{(\max(1, c_C) + c_D)^2} I.
\end{align*}
which concludes the proof.

\section{Numerical Results}\label{sec:numerical}
\paragraph{Simulation Design} 
We performed numerical simulations of the \gls*{arx} model to verify the sample complexity bounds
we obtained theoretically for
distributions with bounded $p$th moments in \Cref{thm:pthmom},
sub-exponential distributions in \Cref{thm:subexp}, 
and sub-Gaussian distributions in \Cref{thm:subgauss}.
We explore an ARX model defined in \cref{eqn:arx} 
with autoregressive order $q=5$ and dimensions $d_y=d_u=10$.
For all simulations, the input signal $u_t=v_t$ follows a Gaussian distribution 
with unit standard deviation,
and the following noise distributions for $w_t$ are considered:
\begin{itemize}
  \item sub-Gaussian distribution: Gaussian distribution,
  \item sub-exponential distribution: Laplace distribution,
  \item Bounded $p$th moment with $p>2$: Student's t distribution ($p=2.1, 4.1$), Pareto distribution ($p=2.1, 4.1$),
  \item No Bounded Moments: Cauchy distribution.
\end{itemize}
We set all noise distributions 
with unit standard deviation except for distributions without finite covariance.
For each noise distribution 
and sample sizes $T = [100, 200, 500, 10^3, 10^4, 4\times 10^4, 10^5, 10^6, 5\times 10^6]$, 
we ran $200$ simulations generating random single-trajectory data.
\gls*{ols} is used to estimate the parameters $\{A_i\}_{i=1}^5$ and $\{B_i\}_{i=1}^5$
with \cref{eqn:theta-hat}.
The error bound of the $99$th, $95$th, and $90$th percentiles of the relative error
are approximated using \texttt{numpy.percentile()}
with the ``higher'' method,
and the slope is approximated using linear fitting or \texttt{numpy.polyfit()}.
\Cref{fig:T} shows the loglog plot of the 
total relative error of the parameters for increasing sample size
plotted against a reference line of slope $-1/2$.

\paragraph{Discussion} 
Our results are shown in \Cref{fig:T}.
Below $1000$ samples, 
$T$ is not sufficiently large to exceed the thresholds in our theoretical results,
so the relative errors are much larger.
To verify the theoretical sample complexity bounds 
when the failure probability $\delta$ is small,
we plot error bounds for the $99$th, $95$th, $90$th percentiles of the relative error, 
where the deviation from the $T^{-1/2}$ rate is not significant.
For Gaussian and Laplace distributions, the sample complexity exhibits a clear $T^{-1/2}$ dependence
compared with the reference line.
As the number of moments decrease, Pareto and Student's t distributions with $p = 2.1$
at the $99$th percentile
exhibit larger variance,
where the fitted slope is $-0.43$ for Student's t and $-0.46$ for Pareto distributions.
While \Cref{thm:pthmom} predicts that the sample complexity upper bound
should deteriorate by $T^{1/p}$,
our numerical results do not clearly indicate whether these distributions actually have slower convergence.
Finally, we explore the sample complexity scaling of distributions without finite covariance.
The Cauchy distribution with no bounded moments exhibit no convergence,
suggesting that the parameters are not estimated properly.

The numerical results indicate that when the covariance is unbounded, 
the empirical error bound can deviate substantially from the $T^{-1/2}$ rate,
such as the Cauchy distribution, which failed to converge. 
When the noise is heavy-tailed with bounded $p$th moments for $p > 2$, 
simulations do not provide clear evidence that the decay rate is slower than $T^{-1/2}$.
We observe a mild deviation from the $T^{-1/2}$ rate for the Student's t
distribution with $p = 2.1$, 
but this behavior may be attributable to rare tail events that fall into the failure
probability of our theoretical bounds. 
We also note that the error decay appears faster than $T^{-1/2}$ for some distributions,
but this could be attributed to sample sizes below the burn-in period.
Thus, our numerical results do not clearly show 
that heavy-tailed distributions result in $T^{1/p}$ deterioration 
of the error bound.

\begin{figure}
  \centering
  \includegraphics[width=1\linewidth]{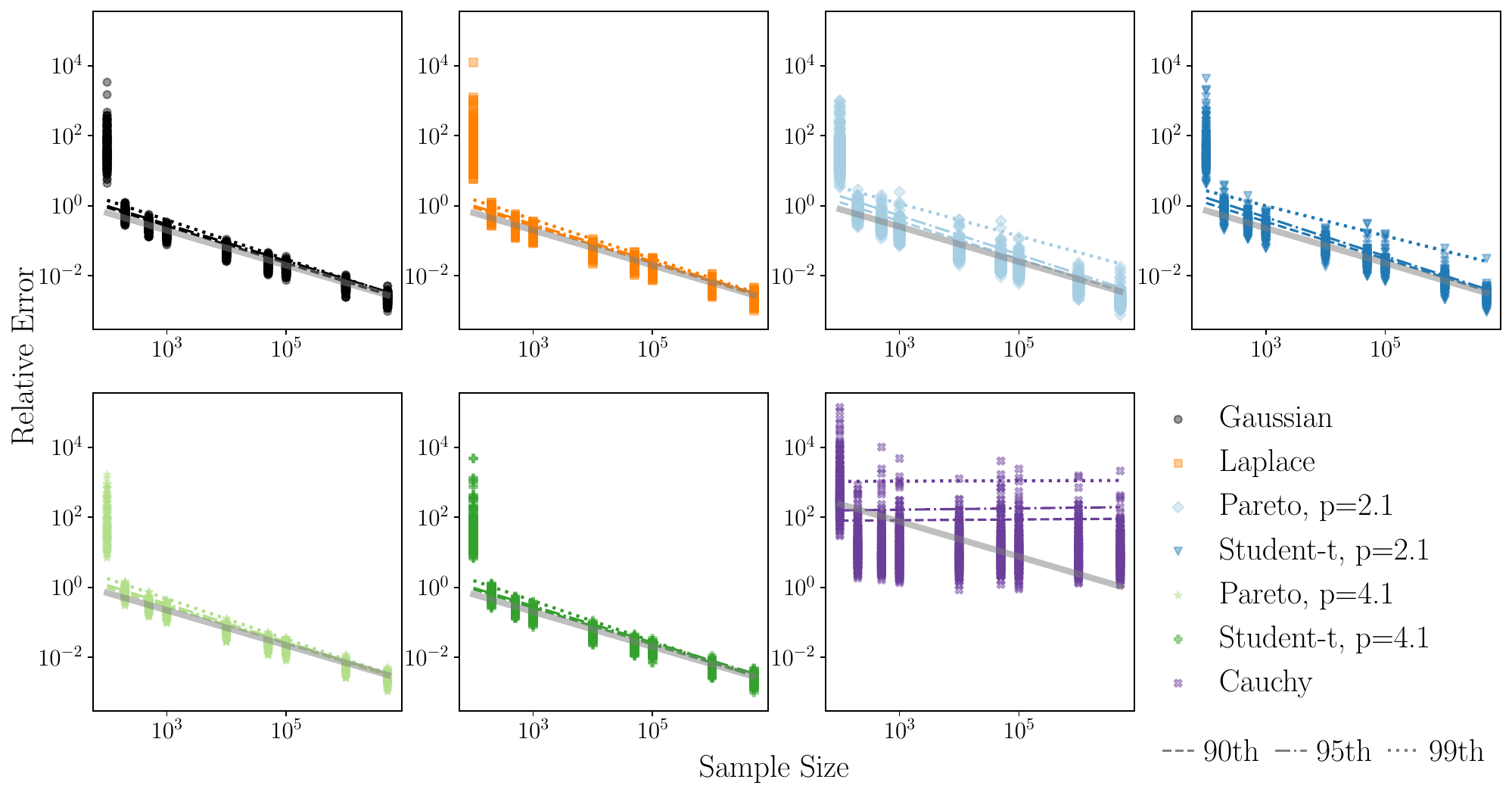}
  \caption{Log log plot of relative parameter estimate error for an ARX Model with OLS estimator 
    of sample sizes $T = [100, 200, 500, 10^3, 10^4, 4\times 10^4, 10^5, 10^6, 5\times 10^6]$.
    The fitted slope for the $99$th, $95$th, and $90$th percentiles are shown in dashed lines,
    and the reference gray line has slope $-1/2$.}
  \label{fig:T}
\end{figure}
\end{document}